\documentclass[10pt,journal]{IEEEtran} \usepackage[T1]{fontenc}
\usepackage{bm} \usepackage{times} \usepackage{amssymb,bbm}
\usepackage{todonotes}
\usepackage[linesnumbered,ruled,vlined]{algorithm2e}
\usepackage{amsthm} \usepackage[english]{babel}
\usepackage{amsmath,textcomp,enumerate,bbm,latexsym}
\usepackage{algpseudocode} \usepackage{graphics} \usepackage{epstopdf}
\usepackage{aliascnt} \usepackage{csquotes} \usepackage{multirow}
\usepackage{subfigure} \usepackage{url} \usepackage{array}
\usepackage{caption} \usepackage{balance}
\usepackage{color}
\usepackage{dblfloatfix}
\newenvironment{enumeratenumeric}{\begin{enumerate}[1.]
  }{\end{enumerate}}
\newcolumntype{C}[1]{>{\centering\let\newline\\\arraybackslash\hspace{0pt}}m{#1}}

\usepackage[colorlinks=false]{hyperref}

\newtheorem{theorem}{Theorem}

\newcolumntype{K}[1]{>{\centering\arraybackslash}p{#1}}

\DeclareMathOperator{\sign}{sign}

\newcommand{\trans}[1]{{#1}^{\ensuremath{\mathsf{T}}}} 

\ifCLASSINFOpdf
\usepackage[nocompress]{cite} \else
\usepackage{cite} \fi

\ifCLASSINFOpdf
\else
\fi

\usepackage{fancyhdr} 
\begin{document}

\title{Robust and Fast Decoding of High-Capacity Color QR Codes for Mobile Applications}
\author{Zhibo~Yang,
        Huanle~Xu,~\IEEEmembership{Member,~IEEE,}
        Jianyuan~Deng,\\
        Chen~Change~Loy,~\IEEEmembership{Senior~Member,~IEEE,}
        and~Wing~Cheong~Lau,~\IEEEmembership{Senior~Member,~IEEE}
\thanks{This work is supported in part by MobiTeC of The Chinese University of Hong Kong and the Innovation and Technology Fund of the Hong Kong SAR Government (Project no. ITS/300/13).}
\thanks{Z. Yang, W. C. Lau and C. C. Loy are with the Department of Information Engineering, The Chinese University of Hong Kong, Hong Kong, China. E-mail: \{zbyang, wclau, ccloy\}@ie.cuhk.edu.hk.}
\thanks{H. Xu is is with the College of Computer Science and Technology, Dongguan University of Technology. E-mail: xuhl@dgut.edu.cn.}
\thanks{J. Deng is with the School of Pharmacy, The Chinese University of Hong Kong, Hong Kong, China. E-mail: jianyuandeng@gmail.com}
\thanks{Part of this work has been presented in IEEE ICIP 2016 \cite{yang2016towards}.}}

\maketitle

\begin{abstract}
The use of color in QR codes brings extra data capacity, but also inflicts tremendous challenges on the decoding process due to chromatic distortion---cross-channel color interference and illumination variation. Particularly, we further discover a new type of chromatic distortion in high-density color QR codes---cross-module color interference---caused by the high density which also makes the geometric distortion correction more challenging.
%
To address these problems, we propose two approaches, LSVM-CMI and QDA-CMI, which jointly model these different types of chromatic distortion. Extended from SVM and QDA, respectively, both LSVM-CMI and QDA-CMI optimize over a particular objective function and learn a color classifier. {Furthermore, a robust geometric transformation method and several pipeline refinements are proposed to boost the decoding performance for mobile applications.} We put forth and implement a framework for high-capacity color QR codes equipped with our methods, called HiQ.
To evaluate the performance of HiQ, 
we collect a challenging large-scale color QR code dataset, CUHK-CQRC, which consists of 5390 high-density color QR code samples. 
The comparison with the baseline method \cite{blasinski2013per} on CUHK-CQRC shows that HiQ at least outperforms \cite{blasinski2013per} by 188\% in decoding success rate and 60\% in bit error rate. Our implementation of HiQ in iOS and Android also demonstrates the effectiveness of our framework in real-world applications. 
\end{abstract}


\IEEEpeerreviewmaketitle

\thispagestyle{fancy}

\section{Introduction}\label{sec:intro}
\IEEEPARstart{I}n recent years, {\color{black}QR codes \cite{liu2008recognition}\cite{soon2008qr}} have gained great popularity because of their quick response to scanning, robustness to damage, readability from any directions. However, the data capacity of existing QR codes has severely hindered their applicability, e.g., adding authentication mechanisms to QR codes to protect users from information leakage \cite{li2015authpaper}. To increase the data capacity of QR codes, leveraging color is arguably the most direct and inexpensive approach.

Unfortunately, it remains an open technical challenge to decode color QR codes in a robust manner, especially for high-density ones.
The difficulties of increasing the capacity/footprint ratio boil down to two types of distortion which can be further exacerbated for high-density color QR codes. One is \textit{geometric distortion}: standard QR code decoding method corrects geometric distortion via perspective projection \cite{hartley2003multiple}, which estimates a projection matrix from four spatial patterns (the so-called finder pattern and alignment pattern) in the four corners of the QR codes. In practice, it is very likely that the estimated positions of the patterns are inaccurate. While small deviation is tolerable when decoding low-density QR codes, perspective projection becomes unreliable for high-density ones as each module (module refers to the small square unit that makes up QR code) only contains few pixels. Consequently, minor errors are amplified and propagated through geometric transformation which ultimately leads to decoding failure.

The other is \textit{chromatic distortion}: For monochrome QR codes, a simple thresholding method is adequate to recover the color since there are only two colors between which the chromatic contrast is often high. However, color recovery for color QR codes, which may consist of 4, 8, or even 16 colors, becomes nontrivial due to chromatic distortion. 
We characterize the chromatic distortion of color QR codes in three different forms based on their physical or optical causes, see Fig. \ref{fig:inter} for illustration:

{\color{black}
\begin{itemize}
\item \textbf{Cross-channel interference (CCI)}. Printing colorants (i.e., C, M and Y colorant layers) tend to interfere with the captured image channels (i.e., R, G and B channels) \cite{blasinski2013per}, see Fig. \ref{fig:cci}. CCI scatters the distribution of each color, and thus leads to difficulties in differentiating one color from the others;
  
\item \textbf{Illumination variation}. Color varies dramatically under different lighting conditions \cite{gijsenij2012improving} (see Fig. \ref{fig:illuvar}). Unfortunately, 
it is inevitable for real-world QR code applications to operate under a wide range of lighting conditions;

\item \textbf{Cross-module interference (CMI)}. For high-density color QR codes, the printing colorants in neighboring data modules may spill over and substantially distort the color of the central module {during the printing process\footnote{{\color{black}It is worth noting that CMI is different from chromatic distortion arisen from the camera side, such as chromatic aberration of lens, motion blur and low image resolution, which  are out of the scope of this work.}}}, see Fig. \ref{fig:cmi}. CMI has negligible influence over low-density color QR codes due to their relatively large data module size. In this case, the cross-module contamination is only limited to the periphery of each module and does not affect its center portion over which color recovery is performed. 

\end{itemize}
}
\begin{figure*}[t]
\centering
\captionsetup{justification=centering,margin=2cm}
\subfigure[Cross-channel color interference.]{
\label{fig:cci}
\includegraphics[width=.475\textwidth]{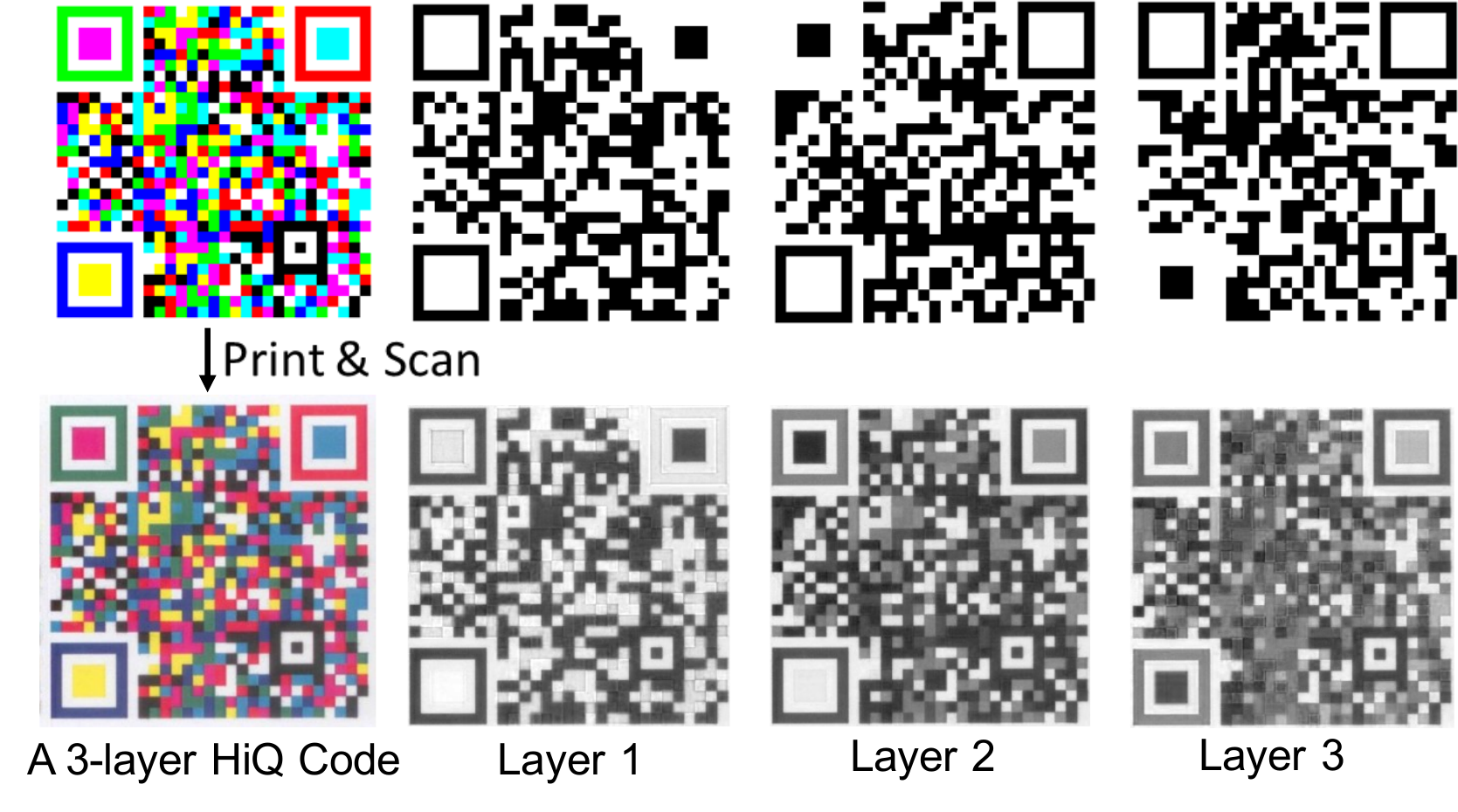}}
 \subfigure[Color distributions of one color QR code under incandescent (left) and outdoor (right) lighting.]{
\label{fig:illuvar}
\includegraphics[width=.3\textwidth]{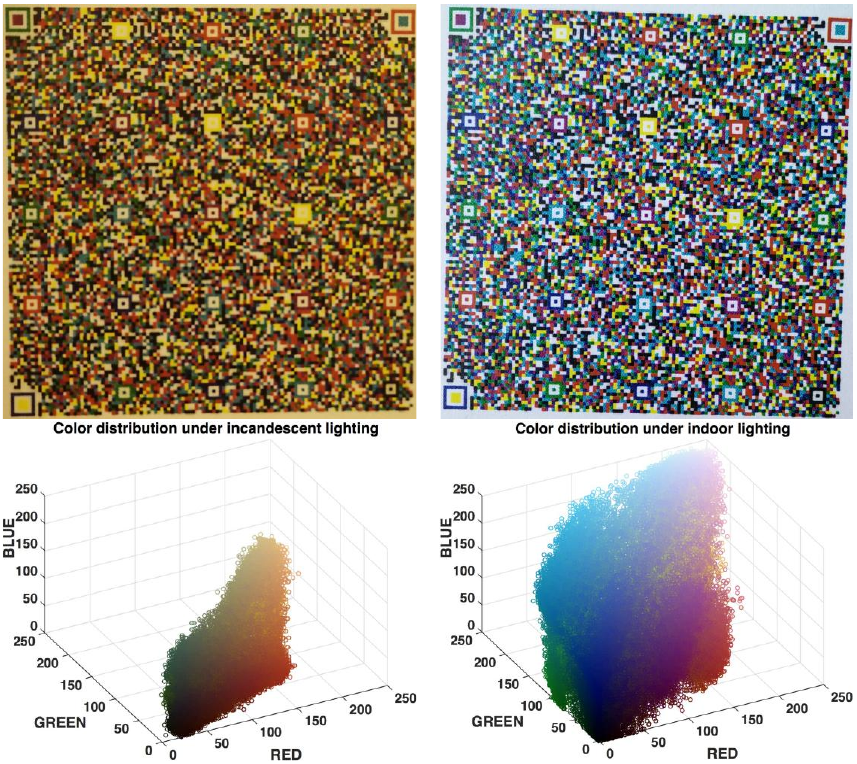}}
\hspace{0.1cm}
\subfigure[Cross-module color interference.]{
\label{fig:cmi}
\includegraphics[width=.15\textwidth]{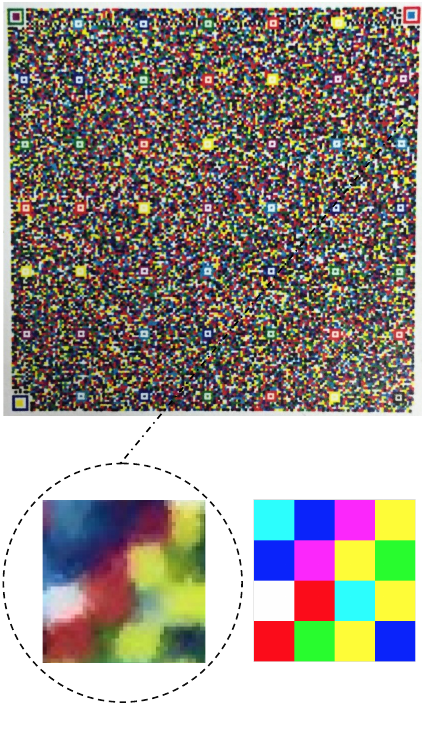}}
\caption{Three types of chromatic distortion of color QR codes.} 
\label{fig:inter}
\end{figure*}

To the best of our knowledge, CMI has never been studied before and is especially important for decoding high-density color QR codes, while CCI and illumination variation have been addressed by prior arts \cite{blasinski2013per}\cite{grillo2010high}. To address illumination variation, they take an online approach, namely, they learn a color recovery model for every captured QR code image. However, we have found that online approach brings huge computational burden to mobile devices and it is difficult to collect enough clean training data for high-density color QR codes due to CMI and other external causes, e.g., dirt, damage and nonuniform illumination on the reference symbols from which the training data are collected. 

In this paper, we adopt an offline learning approach, and model the cross-module interference together with the fallout of illumination variation and the cross-channel interference by formulating the color recovery problem with an optimization framework. In particular, we propose two models, QDA-CMI and LSVM-CMI, which are extended from quadratic discriminant analysis (QDA) and support vector machine (SVM), respectively. A robust geometric transformation method is further developed to accurately correct geometric distortion. Besides, we propose a new color QR code framework, HiQ, which constructs a color QR code by combining multiple monochrome QR codes together in a layered manner to maintain the structure of conventional QR code, and thus to preserve the strength of their design. We refer the color QR codes constructed under HiQ framework as HiQ codes in the remainder of this paper. 
To summarize, this paper has primarily made the following technical contributions:
\begin{itemize}
{\color{black}
\item \textbf{Chromatic distortion correction. }To the best of our knowledge, this paper is the first one that discovers the cross-module color interference in high-density QR codes and establishes models to simultaneously correct different types of chromatic distortion.

\item \textbf{Robust Geometric Transformation and pipeline refinements. }We improve existing geometric distortion correction scheme and propose a robust geometric transformation method for high-density QR codes. We also present several pipeline refinements (e.g., color normalization, spatial randomization and block accumulation) to further boost the decoding performance for mobile applications.}

\item \textbf{Working implementation and applications. }We propose a high-capacity QR code framework, HiQ, which provides users and developers with great flexibility of encoding and decoding QR codes with high capacity. Experimental results show that with HiQ we can encode 2900 bytes, 7700 bytes and 8900 bytes of data in a region as small as ${26\times 26\,\rm{mm}^2}$, ${38\times 38\,\rm{mm}^2}$ and ${42\times 42\,\rm{mm}^2}$, respectively, and can robustly decode the data within 3 seconds using off-the-shelf mobile phone. We release our implementation of one HiQ codes generator\footnote{Available at \url{http://www.authpaper.net/}.} and two mobile decoders on Apple App Store\footnote{iOS App \url{https://itunes.apple.com/hk/app/authpaper-qr-code-scanner/id998403254?ls=1&mt=8}.} and Google Play\footnote{Android App \url{https://play.google.com/store/apps/details?id=edu.cuhk.ie.authbarcodescanner.android}.}.

\item \textbf{A large-scale color QR code dataset. }For benchmarking color QR code decoding algorithms, we create a challenging color QR code dataset, CUHK-CQRC\footnote{Available at \url{http://www.authpaper.net/colorDatabase/index.html}.}, which consists of 5390 samples of color QR codes captured by different mobile phones under different lighting conditions. This is the first large-scale color QR code dataset that is publicly available. We believe many researches and applications will benefit from it.

\end{itemize}

The remainder of this paper is structured as follows. Section \ref{sec:bg_rw} reviews the existing color 2D barcodes systems and motivates the need for a new color QR code system. Section \ref{sec:HiQ} describes the construction of a color QR code under the HiQ framework. {\color{black}Section \ref{sec:decoding} and Section \ref{sec:geotrans} present the details of the proposed models for chromatic distortion correction and geometric transformation, respectively. Additional pipeline refinements for boosting the decoding performance are discussed in Section \ref{sec:impl}}. Section \ref{sec:exp} compares HiQ with the baseline method \cite{blasinski2013per} on CUHK-CQRC. Our implementations of HiQ in both desktop simulation and actual mobile platforms are described to demonstrate the practicality of the proposed algorithms. Section \ref{sec:concl} concludes this paper.

\section{Related Work}\label{sec:bg_rw}
Recent years have seen numerous attempts on using color to increase the capacity of traditional 2D barcodes \cite{blasinski2013per}\cite{grillo2010high}\cite{kato2009novel}\cite{onoda2005hierarchised}\cite{parikh2008localization}\cite{querini2013color} (see Fig. \ref{fig:barcode_examples} for illustration). Besides, color feature has also been imposed on traditional 2D barcodes purely for the purpose of improving the attractiveness of 2D barcodes such as PiCode \cite{chen2016picode}. As a real commercial product, Microsoft High Capacity Color Barcode (HCCB) \cite{parikh2008localization}, encodes data using color triangles with a predefined color palette. However, A. Grillo et. al. \cite{grillo2010high} report fragility in localizing and aligning HCCB codes. The only available HCCB decoder, Microsoft Tag, requires Internet accessibility to support server-based decoding. 

Recent projects like COBRA \cite{hao2012cobra}, Strata \cite{hu2014strata} and FOCUS \cite{hermans2016focus} support visual light communications by streaming a sequence of 2D barcodes from a display to the camera of the receiving smartphone. However, the scope of their work is different from ours. They focus on designing new 2D (color or monochrome) barcode systems that are robust for message streaming (via video sequences) between relatively large smartphone screens (or other displays) and the capturing camera. In contrast, our work focuses on tackling the critical challenges such as CMI and CCI to support fast and robust decoding when dense color QR codes are printed on paper substrates with maximal data-capacity-per-unit-area ratio.  


H. Bagherinia and R. Manduchi \cite{bagherinia2011theory} propose to model color variation under various illuminations using a low-dimensional subspace, e.g., principal component analysis, without requiring reference color patches. T. Shimizu et. al. \cite{shimizu2011color} propose a 64-color 2D barcode and augment the RGB color space using seed colors which functions as references to facilitate color classification. Their method uses 15-dim or 27-dim feature both in training and testing which is prohibitively time-consuming for mobile devices in real-world applications. To decode color barcodes from blurry images, H. Bagherinia and R. Manduchi \cite{bagherinia2014novel} propose an iterative method to address the blur-induced color mixing from neighboring color patches. However, their method takes more than 7 seconds on a desktop with Intel i5-2520M CPU @ 2.50GHz to process a single image, which is completely unacceptable for mobile applications.

Other researchers have extended traditional QR codes to color QR codes in order to increase the data capacity \cite{blasinski2013per}\cite{grillo2010high}\cite{querini2013color}. HCC2D \cite{grillo2010high} encodes multiple data bits in each color symbol and adds extra color symbols around the color QR codes to provide reference data in the decoding process. The per-colorant-channel color  barcodes framework (PCCC) \cite{blasinski2013per} encodes data in three independent monochrome QR codes which represent the three channels in the CMY color space during printing. A color interference cancellation algorithm is also proposed in \cite{blasinski2013per} to perform color recovery. However, both HCC2D and PCCC suffer from the following drawbacks:
\begin{itemize}
\item Parameters of the color recovery model should be learned for every captured image before decoding. Our experiments show such approach not only brings unnecessary computational burden to mobile devices, but also easily introduces bias in the color recovery process since any dirt and damage on the reference symbols, or even nonuniform lighting can easily make color QR codes impossible to decode; 
\item Their evaluations do not study the effectiveness of their proposed schemes on high-density color QR codes\footnote{The color QR samples used by \cite{blasinski2013per} only hold no more than 150 bytes within an area whose size ranges from $15$ to $30\,\rm{mm}^2$.}, neither do they discover or address the problem of cross-module interference; 
\item They do not investigate the limitations regarding smartphone-based implementations.
\end{itemize}

\begin{figure}[t]
\centering
 \subfigure[COBRA code.]{
\label{fig:cobra}
\includegraphics[height=.124\textwidth]{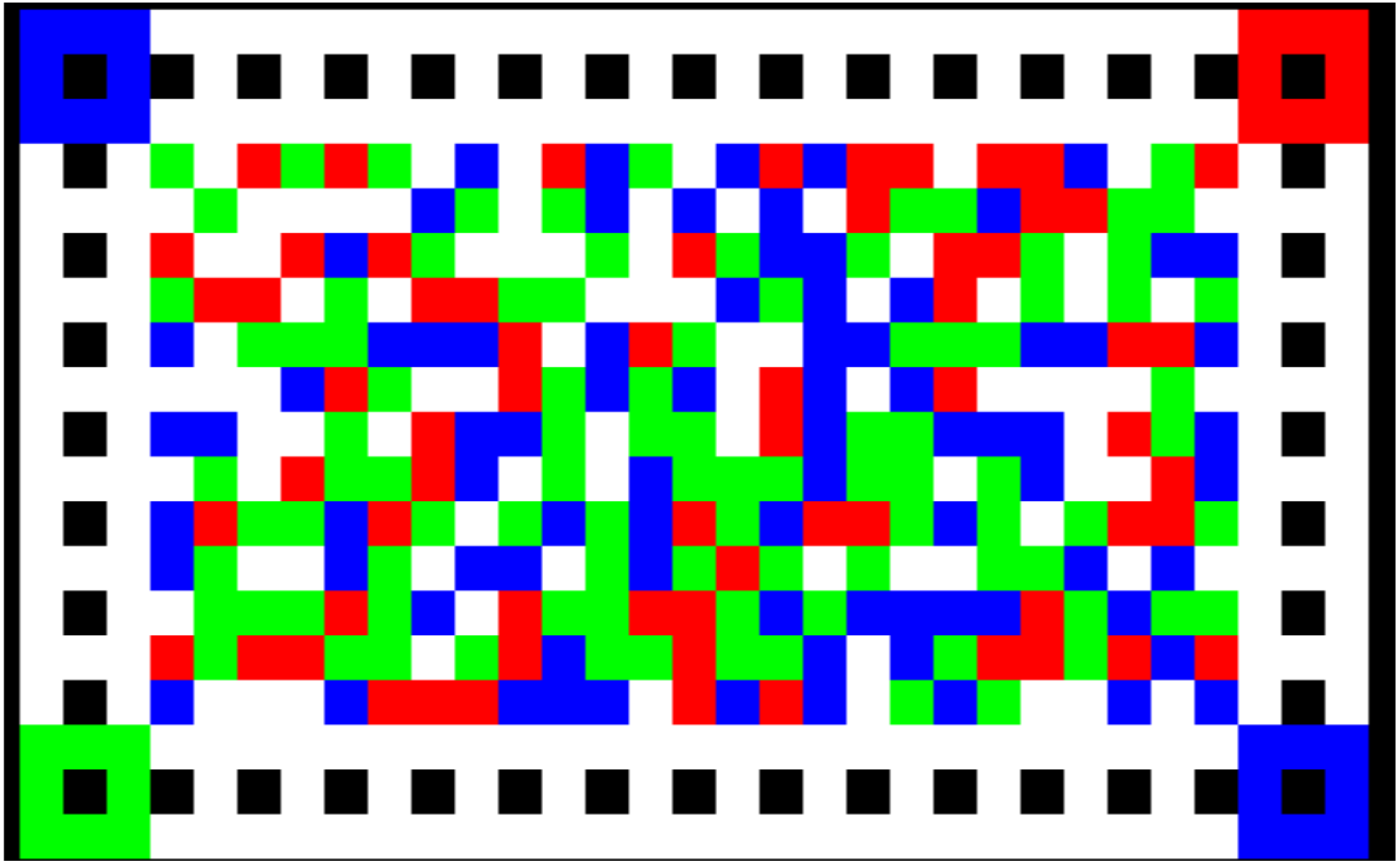}}
\subfigure[Monochrome QR code (green).]{
\label{fig:mqr_c}
\includegraphics[width=.1255\textwidth]{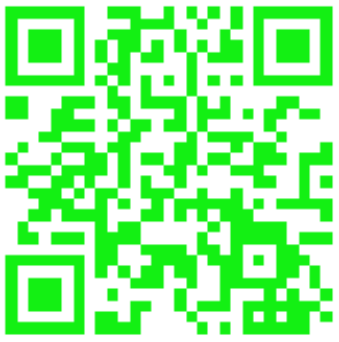}}
\subfigure[High capacity color barcode.]{
\label{hccb}
\includegraphics[width=.1295\textwidth]{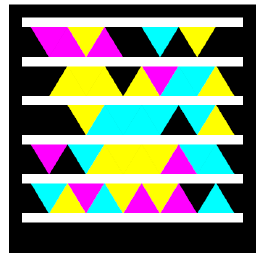}}
\caption{Examples of different types of 2D barcodes.}
\label{fig:barcode_examples}
\end{figure}

In contrast, our proposed HiQ framework addresses the aforementioned limitations in a comprehensive manner. On the encoding side, HiQ differs from HCC2D in that HiQ codes do not add extra reference symbols around the color QR codes; and the color QR codes generation of PCCC framework is a special case of HiQ, namely, 3-layer HiQ codes. On the decoding side, the differences mainly lie in geometric distortion correction and color recovery. HiQ adopts offline learning, {and thus does not rely on the specially designed reference color for training the color recovery model as HCC2D and PCCC do}. More importantly, by using RGT and QDA-CMI (or LSVM-CMI), HiQ addresses the problem of geometric and chromatic distortion particularly for high-density color QR codes which are not considered by HCC2D or PCCC.

\section{HiQ: A Framework for High-Capacity \\QR Codes}\label{sec:HiQ}

\begin{figure*}[t]
\centering
\includegraphics[width=1\textwidth]{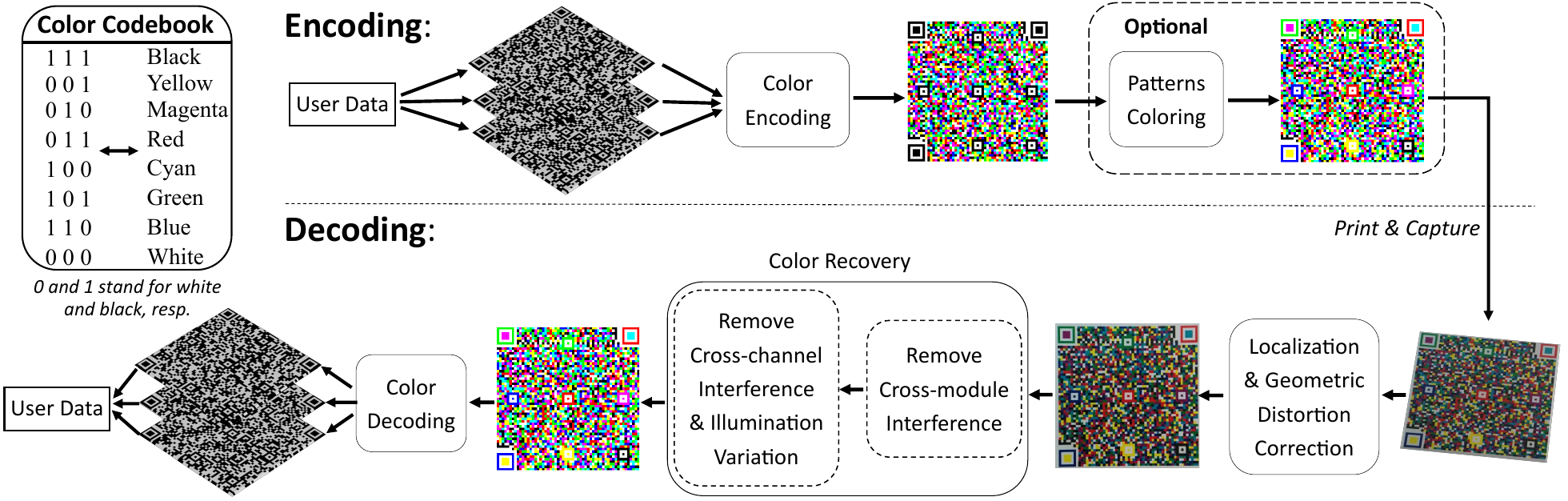}
\caption{{\color{black}An overview of the encoding and decoding of a 3-layer HiQ code in HiQ.} 
}
\label{fig:cqr_encode_decode}
\end{figure*}


Fig.~\ref{fig:cqr_encode_decode} gives an overview of the encoding and decoding process of the proposed HiQ framework. To exploit and reuse existing QR code systems, we keep intact the structure of traditional QR codes in our HiQ code design and select a highly discriminable set of colors to transform multiple traditional QR codes into one HiQ code. Specifically, HiQ firstly partitions the data to be encoded into multiple small pieces and encodes them into different monochrome QR codes independently. Note that different layers of monochrome QR codes can have different levels of error correction, but they must have the same number of modules in order to preserve the structure of conventional QR code. Secondly, HiQ uses different colors that are easily distinguishable to represent different combinations of the overlapping modules of the superposed monochrome QR codes. 
{\color{black} Lastly, the HiQ framework can, as an {\it option}, support  {\bf Pattern Coloring} by painting some special markers (e.g., the Finder and/or Alignment patterns) of the QR code 
with specific colors to either (1) carry extra formatting/ meta information of the HiQ code or (2)  provide reference colors which can be helpful for some existing decoding schemes
 (e.g., PCCC\cite{blasinski2013per})\footnote{Note, however, that the new decoding algorithms proposed in this paper, namely, QDA, LSVM as well as their CMI-extended variants, do not rely on 
 reference-color-painted special patterns/markers for decoding.}.
 Hereafter, we refer to the multiple monochrome QR codes within a HiQ code as its different {\it layers}. 
We call a HiQ code comprised of $n$ monochrome QR codes an  $n$-layer HiQ code.}
 
\begin{figure}[t]
\centering
\subfigure[1-layer QR code, 177-dim.]{
\includegraphics[width=.14\textwidth]{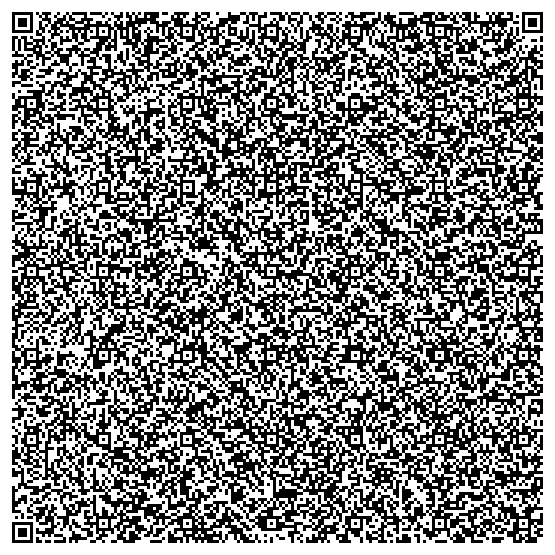}}
\hspace{0.02cm}
\subfigure[2-layer QR code, 125-dim.]{
\includegraphics[width=.14\textwidth]{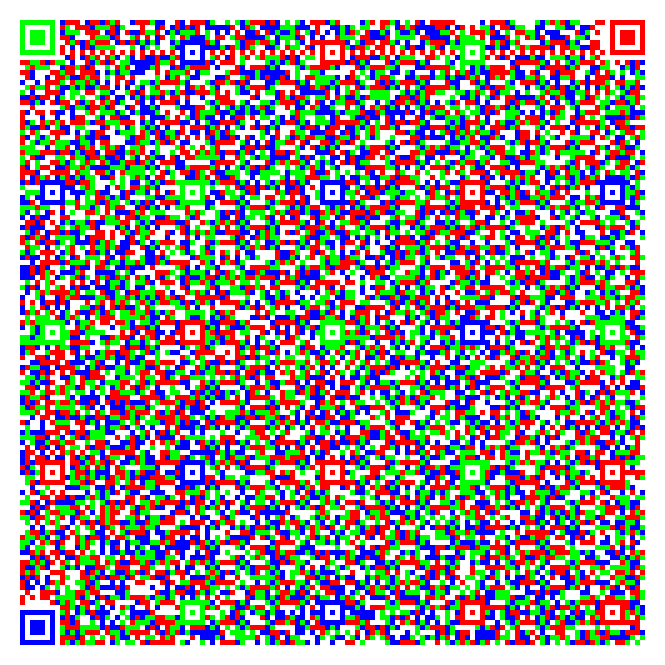}}
\hspace{0.02cm}
\subfigure[3-layer QR code, 105-dim.]{
\includegraphics[width=.14\textwidth]{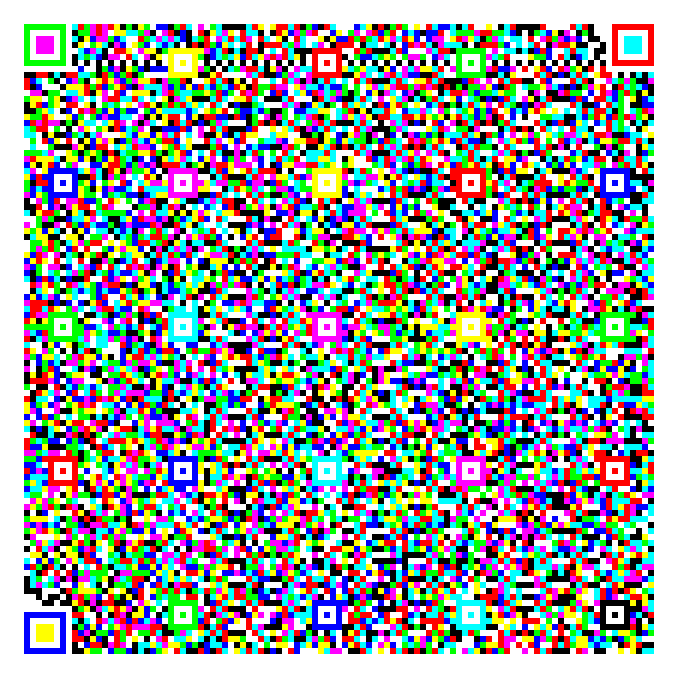}}
\label{fig:cqr3layer}
\caption{Examples of color QR codes of different layers with the same content size. All layers are protected with a low level of error correction.}
\label{fig:cqrlayers}
\end{figure}
Given $n$ monochrome QR codes, $\{M_i\}$, where $i=1,2,\cdots,n$, each $M_i$ is composed of the same number of modules. We denote the $j$th module of $M_i$ by $m_i^j$, where $m^j_i=0$ or 1. In order to achieve layer independence and separability in HiQ codes, HiQ constructs an $n$-layer HiQ code $C_n$ by concatenating all $M_i$ together so that the $j$th module of $C_n$, $c_n^j=\{m^j_1,m^j_2,\cdots,m^j_n\}$. Then, each $c_n^j$ is mapped into a particular color using a predefined color codebook $\mathbb{B}$, where $\lvert \mathbb{B}\rvert = 2^n$ as $m^j_i$ is binary.

An $n$-layer HiQ code has a data capacity that is $n$ times that of a monochrome QR code of the same number of modules. Alternatively, given the same amount of data to carry (within the capacity of a monochrome QR code), HiQ consumes much less substrate footprint in print media than traditional QR code does, assuming same printout density. HiQ codes degrade to monochrome QR codes when $n=1$. The color QR code proposed in \cite{blasinski2013per} is also a special case of HiQ with $n=3$. 
In a nutshell, HiQ is a framework that provides users and developers with more flexibilities in generating QR codes in terms of data capacity, embedded error correction level and appearance (color). Fig. \ref{fig:cqrlayers} gives examples of HiQ color QR codes of different layers ranging from 1 to 3. Given the same amount of user data and printout size, HiQ codes with fewer layers are denser than those with more layers. However, using more layers sharply increases the number of colors in HiQ codes. Consequently, the difficulty of decoding (mainly color recovery) increases.

\section{Modeling Chromatic Distortion}\label{sec:decoding}

\subsection{Framework Overview}

Contrary to other frameworks such as PCCC \cite{blasinski2013per} and HCC2D \cite{grillo2010high} which train the classifier online (in real-time) for each captured image,
HiQ learns the parameters of the classifier offline using color data collected exhaustively from real-world settings of QR codes scanning. This avoids training bias and unnecessary computations on mobile devices. As one of the most commonly used classification tool in many recognition tasks, SVM can be used as a color classifier. To train a multi-class SVM, one-vs-one and one-vs-all \cite{hsu2002comparison} are the widely-adopted schemes. However, they suffer from the drawback that the decoding process is quite time-consuming: one-vs-one and one-vs-all schemes need $2^n$ and $2^{2^n}$ binary classifiers, respectively. Taking advantages of the layered structure in data encoding of HiQ codes, we propose a \textit{layered strategy} where we train a binary SVM for each layer independently to predict the bit in the corresponding layer.

For $n$-layer color QR codes, the training data are denoted as $\{\mathcal{X}, \mathcal{Y}\}$, where $\mathcal{X}$ and $\mathcal{Y}$ are sets of normalized RGB values and binary $n$-tuples (e.g., $\{1,0,\cdots,0\}$), respectively. Traditional one-vs-all strategy just treats $\mathcal{Y}$ as color indicators and trains $2^n$ binary SVMs on $\{\mathcal{X}, \mathcal{Y}\}$ as there are $2^n$ colors. In contrast, we form $n$ binary bit sets, $\mathcal{Y}_1,\mathcal{Y}_2, \cdots, \mathcal{Y}_n$, by separating each element in $\mathcal{Y}$ into $n$ binary indicators, and train $n$ binary SVMs by using $\{\mathcal{X}, \mathcal{Y}_1\},\{\mathcal{X}, \mathcal{Y}_2\}, \cdots, \{\mathcal{X}, \mathcal{Y}_n\}$ as separate sets of training data. In this way, the prediction cost scales \textit{linearly} with the number of layers. We use LSVM as a shorthand for SVM trained using this layered strategy hereafter.

We highlight the following two advantages of LSVM over other traditional methods (e.g., QDA and one-vs-all SVM) in decoding HiQ codes:
\begin{itemize}
\item \textbf{Low processing latency:} One-vs-all SVM requiring $2^n$ binary SVM classifiers, while LSVM only needs $n$ binary SVM classifiers for decoding an $n$-layer HiQ code which is a huge improvement regarding processing latency. 
\item \textbf{Layer separability:} Using LSVM, the classifications of all layers are completely independent. Therefore, in a sequence of scanning of a multi-layer HiQ code, once a layer is decoded it need not be processed anymore, which saves much computational power. In contrast, the predictions of all layers are coupled together in methods like QDA. Thus, even after a layer has been successfully decoded, it will still be redundantly processed until all layers are decoded successfully.
\end{itemize}

In subsequent sections, we extend QDA and LSVM further to tackle cross-module interference as QDA and LSVM are shown to have superior performance in color predication compared with other methods (see \ref{sec:classifierCompare} for detailed results).

\subsection{Incorporating Cross-Module Interference Cancellation}
To address the cross-module interference in high-density HiQ codes, we append the feature of each module---normalized RGB intensities of the central pixel---with that of its four adjacent modules (top, bottom, left and right) to train the color classifier. {\color{black} Two reasons that motivate us to use four adjacent modules instead of eight are: a)  These four adjacent modules that share with the central module same edges where CMI occurs; b) Using eight modules will bring extra nontrivial computational overhead.} However, in this way the feature dimension rises from 3 to 15. In real-world mobile applications, computational power is limited and tens of thousands of predictions per second are required to decode a high-capacity HiQ code which usually consists of ten to thirty thousand modules and it usually takes multiple trials until success. Consequently, directly adding the feature from adjacent modules which increases feature dimension can hardly meet the processing latency requirement. For instance, to decode a $3$-layer HiQ code, our experiences show that if we use QDA as the color classifier, it takes nearly ten seconds for Google Nexus 5 to finish one frame of decoding, which is prohibitively expensive. Moreover, the computational cost grows dramatically as the number of layer increases.

{\color{black}Based on our empirical observations of highly-densed color QR codes, a target central module tends to be corrupted by multiple colors coming from its neighboring modules (and thus the use of the term "cross-module" interference).  Such observations motivate us to make the following key assumption about the cross-module interference based on a simple color mixing rule \cite{simonot2014between}:\textit{The {\bf pre-CMI color} of the central module is a linear combination of the {\bf perceived color} of the central module and that of its four neighboring 
ones.} 
By  {\it perceived color}, we mean the color perceived by the sensor of the decoding camera. 
By {\it pre-CMI color},  we mean the perceived color of a module when there is no (or negligible) cross-module interference.  
Here, each color is represented by a three-dimensional vector in  the RGB color space.
With this assumption, we firstly cancel the CMI to recover the pre-CMI color (3-dimensional) of each module which is a fast linear operation. Secondly, we use the 3-dimensional  pre-CMI color to 
estimate the ground-truth color of the target module instead of using a 15-dimensional input-feature vector to represent the perceived RGB vector of the target module and that of its 4 neighboring modules.
In this way, both accuracy and speed can be achieved. 

\IncMargin{1.1em}
\begin{algorithm}[t]
\caption{Algorithm for solving QDA-CMI}
\label{algo:QDACMI}
\Indm
\KwIn{The training data $\{({X}_i, y_i)\}$ where $y_i\in\{1,\cdots, K\}$}
\KwOut{$\{\boldsymbol{\Sigma}_k\}_{k=1}^K, \{\mu_k\}_{k=1}^K$ and ${\bm{\theta}}$}
\Indp
Initialize ${\bm{\theta}}^0=\trans{[1,0,0,0,0]}$, $j=0$\;
\While{not convergence}{
  \For{$k \in \{1,\cdots, K\}$}{
    $\mu_k^{j+1} = \sum_{i:y_i=k}\trans{{X}_i}{{\bm{\theta}}}^j / N_k$\; 
    $\boldsymbol{\Sigma}_k^{j+1}=\sum_{i:y_i=k}\trans{(\trans{{X}_i}{{\bm{\theta}}}^j-{\mu}_k)}(\trans{{X}_i}{{\bm{\theta}}}^j-{\mu}_k) / N_k$\;
  }
  Compute ${\bm{\theta}}^{j+1}$ using Eq. \eqref{eq:optM}\;
  $j=j+1$\;
}
$\boldsymbol{\Sigma}_k={\Sigma}_k^j, \mu_k=\mu_k^j$, for all $k\in \{1,\cdots, K\}$ and ${\bm{\theta}}={\bm{\theta}}^j$\;
\end{algorithm}

In the following, we represent the color feature of the $i$th sample ${X}_i$ as a $5\times 3$ matrix of which each row is formed by the normalized RGB intensities and define $\bm{\theta}$ as a $5\times 1$ column vector whose items are linear coefficients for the corresponding modules. Thus, the pre-CMI color of the $i$th sample is given by:
\vspace{-.5em}
\begin{equation}\label{eq:assumpeq}
\tilde{x}_i=\trans{{X}_i}\bm{\theta}.
\vspace{-.5em}
\end{equation}
By substituting the training data point, $x_i$, in the formulation of QDA and LSVM with Eq. \eqref{eq:assumpeq}, we introduce two models---QDA-CMI and LSVM-CMI.
} 

\textbf{QDA-CMI. }\label{sec:eqda}  
Using conventional QDA (without considering CMI), we assume the density function of the \textit{perceived color}, $x_i$, to be a multivariate Gaussian: 
\vspace{-.5em}
\begin{equation}
f_k(x_i) = \frac{1}{\sqrt{(2\pi)^l |\boldsymbol{\Sigma}_k|}} e^{-\frac{1}{2}\trans{(x_i-{\mu}_k)} \boldsymbol{\Sigma}_k^{-1} (x_i-{\mu}_k)},
\vspace{-.5em}
\end{equation}
where $k$ is the color (class) index, $l$ is the feature dimension, $\mu_k$ and $\boldsymbol{\Sigma}_k$ are the mean vector and covariance matrix of the $k$th class, respectively.

To incorporate CMI cancellation, we instead model the \textit{pre-CMI color}, $\tilde{x}_i$, rather than the preceived color, as a multivariate Gaussian. Together with Eq. \eqref{eq:assumpeq}, we obtain the following density function:
\vspace{-.5em}
\begin{equation}
f_k(\tilde{x}_i) = \frac{1}{\sqrt{(2\pi)^l |\boldsymbol{\Sigma}_k|}}e^{-\frac{1}{2}\trans{(\trans{X_i} \bm{\theta}-{\mu}_k)} \boldsymbol{\Sigma}_k^{-1}(\trans{X_i} \bm\theta-{\mu}_k)}.
\vspace{-.5em}
\end{equation}
We jointly learn the parameters $\boldsymbol{\Sigma}_k$, $\mu_k$ and $\bm{\theta}$ using maximum likelihood estimation (MLE) which maximizes the following objective function:
\vspace{-.5em}
\begin{equation}\label{eq:mle}
G(\boldsymbol{\Sigma}, \mu, \bm{\theta}) = \prod_{k=1}^K \prod_{i:y_i=k} f_k(\tilde{x}_i),
\vspace{-.5em}
\end{equation}
where $\boldsymbol{\Sigma}=\{\boldsymbol{\Sigma}_k\}_{k=1}^K$ and $\mu=\{\mu_k\}_{k=1}^K$. {\color{black}In this optimization problem, the coefficients in $\bm{\theta}$ do not necessarily need to sum up to one and we can search the whole space of $\mathbbm{R}^5$ to obtain the optimal solution for $\bm{\theta}$.} As such, we solve the optimization problem by alternately optimizing over $(\boldsymbol{\Sigma}, \mu)$ and $\bm{\theta}$. Refer to Algorithm \ref{algo:QDACMI} for details. In the first step, we initialize $\bm{\theta}$ such that the element corresponding to the central module to be 1 and others to be 0. Note that with $\bm{\theta}$ fixed, the problem degenerates to traditional QDA and the solution is the MLE of $K$ multivariate Gaussian distributions: $\mu_k = \sum_{i:y_i=k}\trans{X_i}{\bm{\theta}} / N_k$ and $\boldsymbol{\Sigma}_k=\sum_{i:y_i=k}\trans{(\trans{X_i}{\bm{\theta}}-{\mu}_k)}(\trans{X_i}\bm{\theta}-{\mu}_k) / N_k$, where $N_k$ is the number of class-$k$ observations.

In the second step, we fix $(\boldsymbol{\Sigma}, \mu)$ and optimize $\bm{\theta}$, which is equivalent to maximizing the following log-likelihood function:
\vspace{-.5em}
\begin{align*}
\mathcal{L}&(\boldsymbol{\Sigma}, \mu, \bm{\theta}) = \log G(\boldsymbol{\Sigma}, \mu, \bm{\theta}) \\
&= -\frac{1}{2}\sum_{k=1}^K \sum_{i:y_i=k} {\trans{(\trans{X_i}{\bm{\theta}}-{\mu}_k)} \boldsymbol{\Sigma}_k^{-1}(\trans{X_i}\bm{\theta}-{\mu}_k)} + C,
\vspace{-.5em}
\end{align*}
where $C$ is some constant. By taking the derivative of $\mathcal{L}(\boldsymbol{\Sigma}, \mu, \bm{\theta})$ w.r.t. $\bm{\theta}$ and setting it to zero, we have
\vspace{-.5em}
\begin{equation}\label{eq:optM}
\bm{\theta} = \bigg(\sum_{k=1}^K \sum_{i:y_i=k} \trans{X_i} \boldsymbol{\Sigma}_k^{-1}X_i\bigg)^{-1} \bigg(\sum_{k=1}^K \sum_{i:y_i=k} \trans{X_i} \boldsymbol{\Sigma}_k^{-1}\mu_k\bigg).
\vspace{-.5em}
\end{equation}
Then the algorithm alternates between the first step and the second step until convergence.

\textbf{LSVM-CMI. }\label{sec:layer_scheme}
{\color{black}Similarly, we can also estimate the pre-CMI color using SVM model by substituting Eq. \eqref{eq:assumpeq} into the original formulation of LSVM. In this case, we train the $j$th binary SVM on the training data, $\Xi_j=\{\mathcal{X}, \mathcal{Y}_j\}$ where $1\leq j\leq n$ and then obtain the linear coefficient vector $\bm{\theta}_j = \{\theta_j^1,\cdots,\theta_j^5\}$ to recover the pre-CMI color for each module. As such, this LSVM-CMI model yields the following optimization problem (P1):
\vspace{-.5em}
\begin{align}
  \min_{\bm{\omega}_j, b_j, \bm{\xi}, \bm{\theta}_j}  &  \  \frac{1}{2} \| \bm{\omega}_j \|^2 + C {\sum_{i=1}^{N} \xi_i}  \tag{P1} &\\
   \label{label_positive}
 \mbox{s.t.}  & \  \trans{\bm{\omega}_j} \trans{X_i} \bm{\theta}_j + b_j \geq 1 - \xi_i \quad \forall (x_i,y_i=1)\in\Xi_j, & \\
 \label{label_negative}
  &  \ \trans{\bm{\omega}_j} \trans{X_i} \bm{\theta}_j + b_j \leq -1 + \xi_i \quad \forall (x_i,y_i=0)\in\Xi_j, & \\ 
  \label{constraint_xi}
  &  \xi_i \geq 0, \quad  \forall 1 \leq i \leq N, & \\
  \label{constraint_M_2}
  &   ||\bm{\theta}_j||_k \leq 1. &  
  \vspace{-1em}
\end{align} 
where $||\cdot||_k$ represents the $l_k$-norm of a vector, $\bm{\xi} = (\xi_i | i  = 1,2,\cdots,N)$ and $N$ is the number of training data points. 

Besides the standard constraints of a SVM model,  we have also included  Eq.~\eqref{constraint_M_2}, i.e., $||\bm{\theta}_j||_k \leq 1$, as a constraint of (P1) 
due to the following reason:
Empirical results of the  QDA-CMI model above show that the linear coefficient vector $\bm{\theta_j}$ should not be far away from $[1,0,0,0,0]$. 
{\color{black}For example, the optimal value for $\bm{\theta}$ learned from our QR-code dataset under the QDA-CMI model is given by: $$\bm{\theta}_Q^* = [0.9993, -0.0266, -0.0150, -0.0192, -0.0060].$$ 
 Note from $\bm{\theta}_Q^*$ that, for this dataset, the central module itself indeed makes the dominant contribution to the intensity of the pre-CMI color than the other four neighboring modules.  
This  also indicates that the linear coefficient vector $\bm{\theta}$ to account for CMI should not deviate from $[1,0,0,0,0]$ drastically.} 
Such observations motivate us to add  $||\bm{\theta}_j||_k \leq 1$, as a constraint of (P1) in order to perform a shrinkage operation on $\bm{\theta}_j$.
Like other existing image denoising or classification problems, the choice of $k$ in the shrinkage operation is generally problem dependent. For example, for the classification problem of unordered features in \cite{l1normshrinkage} (Section 10), $l_1$-norm is used while for the other problem of multipath interference cancellation for flight cameras in \cite{l1regularlizationsparsecamera} (Section 4), $l_2$-norm has found to be more effective.
}
{\color{black}
Typically, the use of  $l_1$-norm in the shrinkage operation tends to induce sparsity in the solution space. In the context of our LSVM-CMI problem, we have shown (in the supplementary material of this paper) that only the intensity of one of the 5 modules of interest (i.e. the central target modules and its 4 neighbors) would contribute to the pre-CMI color of the central module. By contrast, the use of $l_2$-norm shrinkage would reduce the variance across all coefficients and thus lead to a non-sparse result. Refer to the supplementary material of this paper for a detailed analysis and comparison of the effectiveness of using $l_1$-norm vs $l_2$-norm shrinkage for the LSVM-CMI problem. The conclusion therein drives us to choose $k=2$ for Eq.~\eqref{constraint_M_2} and Constraint \eqref{constraint_M_2} becomes:
\begin{equation}
\label{second_order_constraint}
||\bm{\theta}_j||_2 \leq 1.
\end{equation}

}


{\color{black}In the rest of this section, we shall solve the optimization problem P1 subject to the constraint in Eq. \eqref{second_order_constraint}.}
{\color{black}
Observe that, on the one hand, for a fixed $\bm{\theta}_j$, P1 reduces to a standard SVM optimization problem. On the other hand, when $\bm{\omega}_j$ and $b_j$ are given, P1 is equivalent to the following optimization problem (P3):
\vspace{-.2em}
\begin{align}
\min_{\bm{\theta}_j}  &  \sum_{i=1}^{N} \max \Big\{0,1 + (1 - 2y_i)(\trans{\bm{\omega}_j} \trans{X_i} \bm{\theta}_j + b_j) \Big\}  \tag{P3} &\\
   \label{M2_another}
\mbox{s.t.}  &  \ ||\bm{\theta}_j||_2 \leq 1. 
\vspace{-1.5em}
\end{align} 
\vspace{-1em}
}

{\color{black}
P3 is a convex optimization problem and we adopt the gradient projection approach \cite{nonlinear_programming} to seek the optimal solutions. The corresponding pseudo-code of our designed algorithm is exhibited as Algorithm \ref{algo:LSVMCMI}. 
}

\begin{algorithm}[h!]
\begin{enumeratenumeric}
  \item Initialize $\bm{\theta}_j = \{1,0,0,0,0\}$;
  
  \item Repeat the following steps until convergence;
  
  \item Fix ${\bm{\theta}_j}$, apply the dual approach to solve P1, which outputs the local optimal solution $\bm{\omega}_j$ and $b_j$;
  
  \item Fix $\bm{\omega}_j$, apply the gradient projection approach to solve P3, which outputs the local optimal solution $\bm{\theta}_j$;
\end{enumeratenumeric}
\caption{Algorithm for solving LSVM-CMI}
\label{algo:LSVMCMI}
\end{algorithm}

{\color{black}
We characterize the convergency of Algorithm \ref{algo:LSVMCMI} in the following theorem:

\begin{theorem}
Algorithm \ref{algo:LSVMCMI} converges to a local optimum of P1. 
\end{theorem}
\begin{proof}
In Algorithm \ref{algo:LSVMCMI}, the optimization process in Step 3 and Step 4 shall only decrease the objective value, Algorithm \ref{algo:LSVMCMI} converges to a local optimum of P1. 
\end{proof}

Theorem 1 states that, within finite steps, Algorithm \ref{algo:LSVMCMI} shall output a local optimal solution to P1, i.e., $w_j^*$, $b_j^*$ and $\bm{\theta}_j^*$. So given a testing sample $x$, we use all $n$ SVMs to output the predicted $n$-tuple, $\hat{\mathbf{y}}=\{\hat{y_1},\cdots,\hat{y_n}\}$, where 
}

\begin{equation}
\hat{y_j}=\sign(\trans{(\bm{\omega}_j^*)} X \bm{\theta}_j^*+b^*_j).
\end{equation}

Experimental results on 5390 3-layer HiQ code samples show that CMI cancellation reduces the decoding failure rate by nearly 14\% (from 65\% to 56\%) and reduces the bit error rate from 4.3\% to 3.2\% averaging over all layers. See Section \ref{sec:cqrc_eva} for more detailed evaluation which also reveals that CMI cancellation helps to increase the density of HiQ codes. For example, with LSVM-CMI the minimum decodable printout size of a 7700-byte HiQ code is reduced from $50\times50\,\rm{mm}^2$ to $38\times38\,\rm{mm}^2$ (the density is increased by more than 46\%).

\section{Robust Geometric Transformation}\label{sec:geotrans}


Standard methods correct geometric distortion by detecting four spatial patterns in the corners of a QR code. However, in practice, the detection is inevitably inaccurate, and the cross-module interference makes the decoding more sensitive to transformation errors caused by inexact detection, especially for high-density QR codes. We find that using more points to calculate geometric transformation reduces the reconstruction error significantly, see Fig. \ref{fig:varypixelnum} for experimental results. Therefore, instead of developing a more complicated detection algorithm which increases processing latency, we address this problem by using a robust geometric transformation (RGT) algorithm which accurately samples for each module a pixel within the central region where the color interference is less severe than that along the edges of a module. Unlike standard methods, RGT leverages all spatial patterns, including the internal ones, and solves a \textit{weighted} over-determined linear system to estimate the transformation matrix.

\begin{figure}[t]
\centering
\subfigure[The effect of adding more labeled points.]{
\label{fig:varypixelnum}
\includegraphics[width=.232\textwidth]{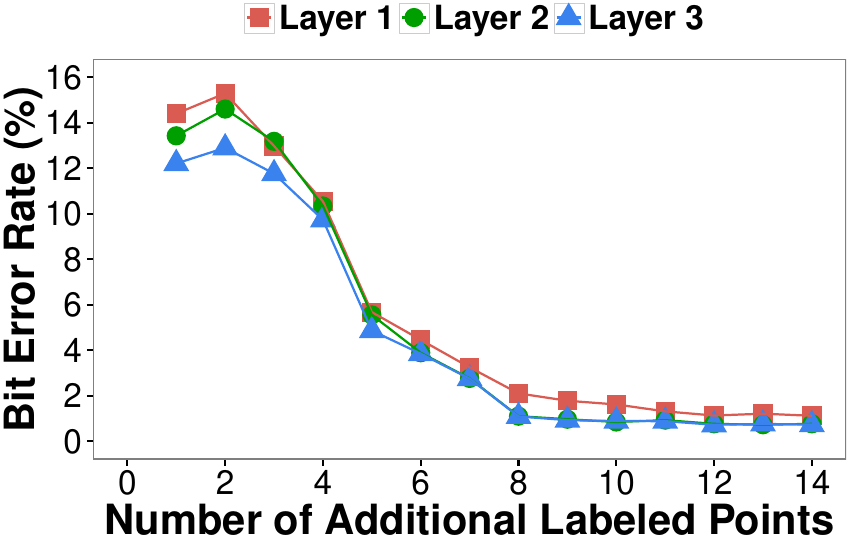}}
\subfigure[The effect of shifting one labeled point from the correct location.]{
\label{fig:shiftpixel}
\includegraphics[width=.232\textwidth]{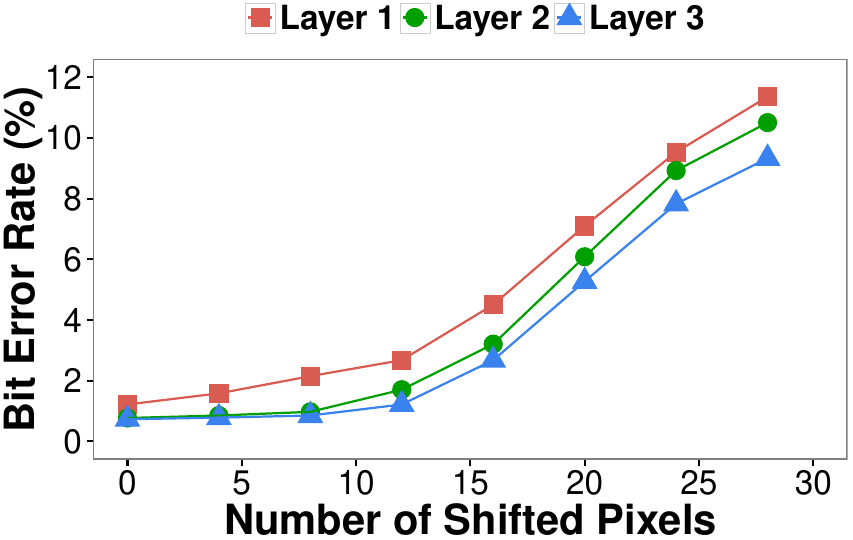}}
\caption{The estimation of the geometric transformation matrix. 
The module size of the HiQ code under testing is approximately 17 pixels.} 
\label{fig:rgtexp}
\vspace{-1em}
\end{figure}

Given $N$ tuples each of which consists of a pair of 2D data-points, namely, ${\{<\mathbf{x_i},\mathbf{x_i}'>,i=1,2,\cdots,N\}}$, where $\mathbf{x_i}$ is the position of a detected pattern, and $\mathbf{x_i}'$ is the corresponding point in the data matrix to be reconstructed. In perspective projection \cite{hartley2003multiple}, $\mathbf{x_i}$ is the homogeneous coordinate representation $(x_i, y_i, z_i)$ where we empirically choose $z_i=1$, and each pair of corresponding points gives two linear equations:
\begin{equation} 
\label{eq_homography}
A_i\mathrm{H}=\mathbf{0},
\end{equation}
where $\mathrm{H}$ is the transformation matrix to be estimated and
$$A_i=\begin{bmatrix} \trans{\mathbf{0}} & -\trans{\mathbf{x_i}} & y_i'\trans{\mathbf{x_i}} \\ \trans{\mathbf{x_i}} & \trans{\mathbf{0}} & -x_i'\trans{\mathbf{x_i}} \end{bmatrix}.$$
Note that although $\mathrm{H}$ has 9 entries, $h_1, h_2, \cdots, h_9$, since in 2D homographies, $\mathrm{H}$ is defined up to scale, and thus has eight degrees of freedom and one may choose $h_9=1$. Therefore, four point coordinates give eight independent linear equations as Eq.\eqref{eq_homography} which are enough for estimating $\mathrm{H}$. However, since the estimated positions of the special patterns often contain noise, which implies $A_i\mathrm{H}\neq \mathbf{0}$, RGT regards the norm $\|A_i\mathrm{H}\|_2$ as the transformation error and minimizes a \textit{weighted} error sum to obtain the estimation of $\mathrm{H}$:
\begin{equation}
\label{eq_opt1}
\begin{aligned}
& \underset{\mathrm{H}}{\text{minimize}}
& & \sum_{i=1}^Nw_i\|A_i\mathrm{H}\|_2 \\
& \text{subject to}
& & \|\mathrm{H}\|_2=1,
\end{aligned}
\end{equation}
where $w_i$ is the weighting factor of each input point $\mathbf{x_i}$. Instead of arbitrarily fixing one $h_i$, we add constraint $\|\mathrm{H}\|_2=1$ to avoid $\mathrm{H}=\mathbf{0}$. As we find that the estimated positions of finder patterns are often more accurate than that of alignment patterns, we assign higher weights to detected positions of finder patterns and lower weights to alignment patterns. Empirically, we set $w_i=0.6$ if $\mathbf{x_i}$ is from the finder pattern, $w_i=0.4$ otherwise. Note that solving \eqref{eq_opt1} is equivalent to solve the following unconstrained optimization problem:
\begin{equation}
\label{eq_opt2}
\begin{aligned}
& \underset{\mathrm{H}}{\text{minimize}}
& & \frac{\|\mathrm{A}\mathrm{H}\|_2}{\|\mathrm{H}\|_2}
\end{aligned}
\end{equation}
where $\mathrm{A}$ is a matrix built from $\{w_iA_i\rvert i=1,\cdots,N\}$, and each $w_iA_i$ contributes two matrix rows to $\mathrm{A}$. Fortunately, the solution to Eq.\eqref{eq_opt2} is just the corresponding singular vector of the smallest singular value \cite{hartley2003multiple}. Singular-value decomposition (SVD) can be used to solve this problem efficiently.

As is shown in Fig. \ref{fig:shiftpixel}, RGT is robust to \textit{minor} shift in the detected positions, but \textit{not} false positives. To reduce false positives, we take advantage of the color property by coloring each pattern with a specific color in the encoding phase (see Fig. \ref{fig:cqr_encode_decode}). For each detected pattern, we filter out possible false detections by checking whether the color of it is correct or not. {\color{black}We demonstrate the effectiveness of RGT by comparing the baseline PCCC with and without RGT in Fig. \ref{fig:rst_DSR_BER} (see Sec. \ref{sec:cqrc_eva} for details).}

{\color{black}
\section{Additional Refinements and Performance Optimization}\label{sec:impl}

\subsection{Color Normalization}
The problem of illumination variations gives rise to the so-called color constancy \cite{cheng2015effective}\cite{shi2016deep} problem which has been an active area of computer vision research. However, most existing algorithms for color constancy tend to be computation-intensive, and thus are not viable for our application of HiQ code decoding using off-the-shelf smartphones. To balance between complexity and efficacy, we adopt the method from \cite{gijsenij2012improving} and normalize the RGB intensities of each sampled pixel with the white color estimated from the QR code image by leveraging its structure. This effectively makes the color feature less illumination-sensitive.

Given a captured image of an $n$-layer HiQ code, we first estimate the RGB intensities of the white color $\mathbf{W}$, of the captured image from white regions in the HiQ codes (e.g., white areas along the boundaries and within the spatial patterns). We denote a pixel sampled during geometric transformation\footnote{\color{black}Our method samples one representative pixel from the central region of each module.} by $(\boldsymbol{x},y)$, where $\boldsymbol{x}$ is a 3-dim color feature and $y=1,2,\cdots,2^n$ being the color label. Instead of directly using RGB intensities, $\mathbf{I}=\{\mathbf{I}_R, \mathbf{I}_G, \mathbf{I}_B\}$, as the color feature for color recovery, we normalize $\mathbf{I}$ by $\mathbf{W}$: $\boldsymbol{x}_j=\mathbf{I}_j / \mathbf{W}_j, j\in\{R,G,B\}$.

Yet due to the fact that the estimation of white color (i.e., $\mathbf{W}$) may contain noise, we adopt the data augmentation technique commonly used in training neural networks \cite{luke2017augmentation} and augment the training data by deliberately injecting noise to $\mathbf{W}$ to enhance the robustness of the color classifier. More precisely, besides the original data point $(\boldsymbol{x},y)$, each sampled pixel $\mathbf{I}$ is further normalized by five \enquote{noisy} estimations of white color which are randomly and independently drawn from a normal distribution with mean $\mathbf{W}$ and a small standard deviation. It is worth noting that the color normalization does not suffer from the problems caused by the use of reference color like other methods (discussed in Sec. \ref{sec:intro}) because sampling very few pixels from the white regions will suffice and it is resilient to estimation noise.}

\subsection{Local Binarization}\label{sec:binarize}

Existing monochrome QR code decoders usually use image luminance, e.g., the Y channel of the YUV color space, to binarize QR codes. However, directly applying it on color ones can be problematic {\color{black}because some colors have much higher luminance than other colors (e.g., yellow is often binarized as white)}, which makes some patterns undetectable. To solve this problem, we use a simple but effective method to binarize HiQ codes. Let $\mathbf{I}$ denotes an image of a HiQ code formatted in the \textit{RGB} color space. We first equally divide it into $8\times 8$ blocks. In each block, a threshold is computed for each channel as follows:
$$T_i = \frac{\max(\mathbf{I}_i)+\min(\mathbf{I}_i)}{2}$$
where $i\in\{R, G, B\}$ and $\mathbf{I}_i$ is the $i$th channel of image $\mathbf{I}$. A pixel denoted by a triplet $(P_R,P_G,P_B)$ is assigned 1 (black) if $P_i<T_i$ for any $i\in\{R, G, B\}$, 0 (white) otherwise.

\begin{figure}[t]
\centering
\includegraphics[width=.35\textwidth]{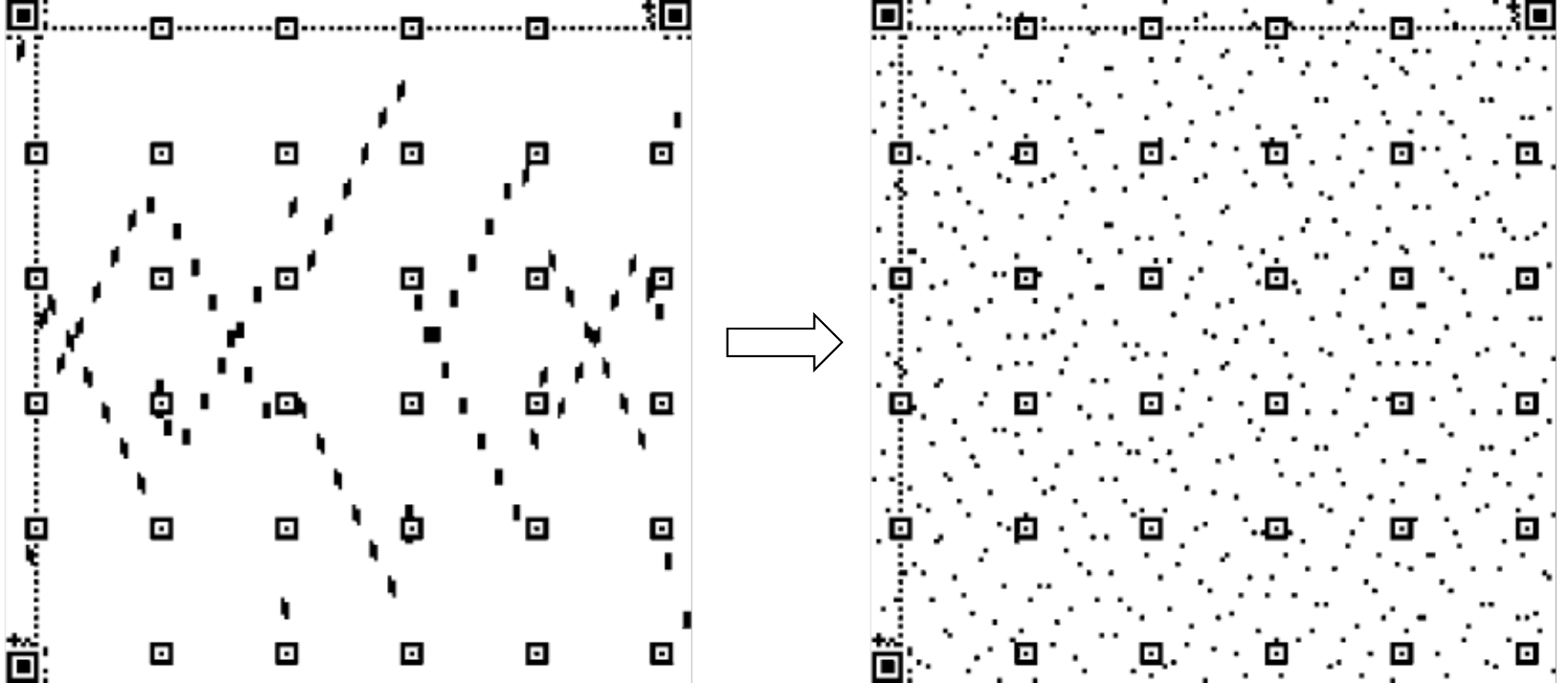}
\caption{Spatial Randomization of Data Bits. On the left is the bit distribution of the original QR code, and on the right is the bit distribution after randomization.}
\label{randomization}
\vspace{-1em}
\end{figure}

\begin{figure*}[t]
\centering
\includegraphics[width=0.9\textwidth]{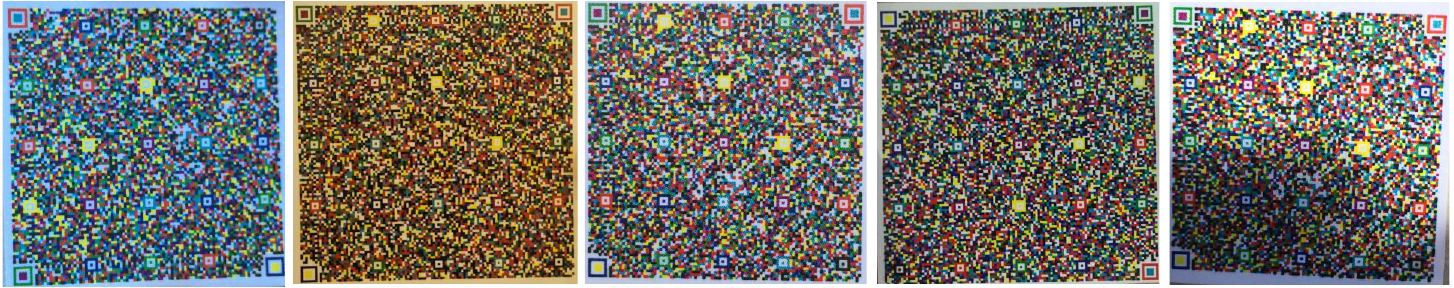}
\caption{Samples from CUHK-CQRC captured under different lighting conditions.}
\label{CQRC_smp}
\end{figure*}

\subsection{Spatial Randomization and Block Accumulation}\label{sec:randaccu}
Through our experiences of scanning QR codes using mobile devices, we noticed one strange fact that some localized region somehow causes the entire decoding process to fail even though the bit error rate averaging over the entire QR code should be recoverable by the built-in Reed-Solomon error correcting coding of the QR code. After examining the error correction mechanism, we surprisingly found that QR code decoder performs error correction block by block, and in each data block, data is shuffled byte by byte in QR codes. However, for \textit{high-density QR codes}, data bits from the same block do not spread out uniformly. Instead, they tend to assemble in the local areas (see Fig. \ref{randomization} for illustration). Consequently, the concentration of the data bits in some specific blocks easily leads to error-correction failure because external factors like local overexposure often lead to a large error percentage in one block beyond repair. With even a single block failure, the decoder will initiate a new round of scanning while discarding all information successfully decoded from other blocks. Moreover, in most failure decoding cases, errors in each captured image always assemble in few blocks instead of affecting all the blocks.

To improve the performance of HiQ framework and reduce scanning latency, we make the following adaptations:
\begin{itemize}
\item \textbf{Spatial randomization}: To avoid data block decoding failure caused by local errors, we propose to shuffle the data of each block bit-by-bit into the whole matrix \textit{uniformly}, which we call as \textit{spatial bits randomization}, to improve the probability of successful correction of each block as shown in Fig. \ref{randomization}.

\item \textbf{Data block accumulation}: Moreover, in order to prevent the failure in one single block which makes the efforts in other blocks in vain, we propose \textit{Data Block Accumulation}, which accumulates the successfully decoded data blocks in previous scans until all the data blocks are decoded.
\end{itemize}
By using these refinements, we manage to cut down the scanning latency significantly, see Section \ref{sec:exp_mobile} for experimental results.

\section{Evaluation}\label{sec:exp}
In this section, we present the implementation details of HiQ and the results of comprehensive experiments. We design two phases of experiments. In the first phase, we compare HiQ with the baseline method on a challenging HiQ code dataset, CUHK-CQRC. 
In the second phase, we evaluate HiQ in real-world operation using off-the-shelf mobile devices. In particular, we collect a large-scale HiQ code dataset, CUHK-CQRC, to evaluate the performance of HiQ by comparing with PCCC \cite{blasinski2013per}.
{\color{black}
  For a fair comparison, we generate HiQ codes with {\bf Pattern Coloring} because PCCC needs reference color for decoding, but our decoding schemes do not need Pattern Coloring.
  }
Note that PCCC provides two different methods: Pilot block (PB) for color QR codes with embedded reference color and EM for those without using the reference color.
In our implementation of PCCC we use PB because PB is reported to outperform EM \cite{blasinski2013per}.

\subsection{Performance Metrics}\label{sec:metrics}
We use the following three metrics to quantify the performance of each approach: 1) Bit Error Rate (BER), 2) Decoding Failure Rate (DFR), 3) scanning time. DFR and BER are used in the evaluation on CUHK-CQRC which is conducted via Matlab simulation in a frame-by-frame manner. Scanning time is used for characterizing the overall user-perceived performance under practical settings.

BER denotes the percentage of wrongly decoded bits \textit{before} applying the built-in Reed-Solomon error correction mechanism; DFR is the percentage of the QR codes that cannot be decoded \textit{after} the error correction mechanism is applied over those that can be successfully localized. DFR measures the overall performance of the decoding method. Compared with DFR, BER is a more fine-grained metric which directly measures the error of color recovery and geometric transformation.
Scanning time is the interval between the time when the camera takes the first frame and the time when the decoding is successfully completed. It measures the overall performance of the decoding approaches on mobile devices which quantifies the user experience.

\subsection{CUHK-CQRC: A Large-Scale Color QR Code Dataset}\label{CQRC}
\begin{table}[b]
\caption{Types of smartphones used in collecting database}
\begin{center}
\begin{tabular}{p{0.3cm}p{2.1cm}p{1.3cm}p{2cm}p{0.5cm}} 
\hline
\textbf{ID} & \textbf{Modle Name} & \textbf{Megapixels (MP)} & \textbf{Image Stabilization} & \textbf{Auto-focus} \\ 
\hline
\hline
1 & iPhone 6 plus & 8.0 & \checkmark (optical) &  \checkmark \\ 
2 & iPhone 6 & 8.0 &  \checkmark (digital) &  \checkmark \\ 
3 & Nexus 4 & 8.0 &   &  \checkmark \\ 
4 & Meizu MX2 & 8.0 &  &  \checkmark \\ 
5 & Oneplus 1 & 13.0 &  &  \checkmark \\ 
6 & Galaxy Nexus 3 & 5.0 &  &  \checkmark \\ 
7 & Sony Xperia M2 & 8.0 &  &  \checkmark \\ 
8 & Nexus 5 & 8.0 & \checkmark (optical) & \checkmark \\
\hline
\end{tabular}
\end{center}
\label{table:phone}
\end{table}
We establish a challenging HiQ code dataset, CUHK-CQRC, in this paper. CUHK-CQRC consists of 1,506 photos and 3,884 camera previews (video frames) of high-density 3-layer color QR codes captured by different phone models under different lighting conditions. Fig. \ref{CQRC_smp} presents some samples of CUHK-CQRC. Different from \cite{blasinski2013per}, we also include previews in our dataset because of the following two reasons. Firstly, photos are different from previews. {\color{black}When users take a photo using the on-board camera of a mobile phone, many embedded systems implicitly process (e.g., deblurring, sharpening, etc) the output image in order to make it more attractive in appearance, while preview may not go through this process. When compared with the captured images, previews are often of a lower resolution.} Secondly, compared with capturing \textit{photos}, it is much faster and more cost-effective for a mobile phone camera to generate previews. Hence, most mobile applications use camera previews as the input of the decoder. 

{\color{black}
\begin{table*}[t!]
{\color{black}
\caption{\color{black}Comparison between 2-layer and 3-layer HiQ codes}
\begin{center}
\begin{tabular}{ m{2.8cm}| K{1.2cm} K{1.2cm} K{1.2cm} K{1.2cm} K{1.8cm} K{1.8cm} }
\hline
\multirow{2}{*}{}  & \multicolumn{2}{ c }{\textbf{2900 bytes}} & \multicolumn{2}{ c }{\textbf{4500 bytes}}
& \textbf{5800 bytes} & \textbf{8900 bytes} \\
\cline{2-7}
& \textbf{2-layer} & \textbf{3-layer} & \textbf{2-layer} & \textbf{3-layer} & \textbf{2-layer} & \textbf{3-layer} \\
\hline
\hline
\textbf{QR code dimension} & 125 & 105 & 157 & 125 & 177 & 177 \\ 
\hline
\textbf{Limit size (cm)} & 2.5 & 2.6 & 3.5 & 3.4 & 3.8 & 5.8 \\ 
\hline
\textbf{Number of module} & 15,625 & 11,025 & 24,649 & 15,625 & 31,329 & 31,329 \\ 
\hline
\textbf{Predictions per frame} & 62,500 & 88,200 & 98,596 & 125,000 & 125,316 & 250,632 \\ 
\hline
\end{tabular}
\end{center}
\label{table:layerdensitytradeoff}
}
\end{table*}
}

\begin{table}[t]
\caption{Color prediction under fluorescent light accuracy of different methods}
\begin{center}
\begin{tabular}{p{2cm}ccccc} 
\hline
\textbf{Method (kernel)} & \textbf{Layer 1} & \textbf{Layer 2} & \textbf{Layer 3} & \textbf{Avg} & \textbf{Time}\\ 
\hline
\hline
LSVM (linear) & 0.35\%& 0.66\%& 4.07\%&  1.69\% & 1\\
SVM (linear) & 1.72\%& 0.71\%& 3.16\%& 1.86\% & 2.7\\
LSVM (RBF) & 0.29\%& 0.56\%& 1.85\%&  0.90\% & $\infty$\\
SVM (RBF) & 0.38\%& 1.68\%& 2.01\%&  1.02\% & $\infty$\\
{QDA} & 0.32\%& 0.60\%& 1.86\%&  0.93\% & 10.7\\
Decision Forest & 0.55\%& 1.47\%& 3.07\%&  1.70\% & $\infty$\\%
{LSVM (Poly-3)} & 0.28\%& 0.57\%& 2.00\%&  0.95\% & 6.7\\
LSVM (Poly-2) & 0.32\%& 0.59\%& 2.60\%&  1.17\% & 3.3\\
\hline
\multicolumn{6}{c}{\footnotesize{\enquote{$\infty$} means the algorithm is too heavy-weight for mobile implementation.}}
\end{tabular}
\end{center}
\label{table:layeredscheme}
\end{table}

We implement the HiQ code generator based on an open-source barcode processing library, ZXing. For fair comparison between HiQ and PCCC where the proposed color QR codes are inherently 3-layer, we generate 5 high-capacity 3-layer color QR codes with different data capacities (excluding redundancies from error correction mechanism) which are 2787 bytes, 3819 bytes, 5196 bytes, 6909 bytes and 8859 bytes (maximum for a 3-layer HiQ code). In order to test the limit of each approach, all color QR codes are embedded with \textit{low} level of error correction in each layer. By using a common color printer (Ricoh Aficio MP C5501A), we print each generated HiQ code on ordinary white paper substrates in different printout sizes, $30\,mm$, $40\,mm$, $50\,mm$ and $60\,mm$ (for simplicity, we use the length of one side of the square to represent the printout size), and in two different printout resolutions, $600\,dpi$ and $1200\,dpi$. To simulate the normal scanning scenario, the samples are captured by different users under several typical lighting conditions: indoor, outdoor (under different types of weather and time of a day), fluorescent, incandescent, and shadowed (both uniform and nonuniform cases are considered). Moreover, we capture the images using eight types of popular smartphones (see Table \ref{table:phone} for details).

\begin{figure*}[t]
\centering
\includegraphics[width=0.95\textwidth]{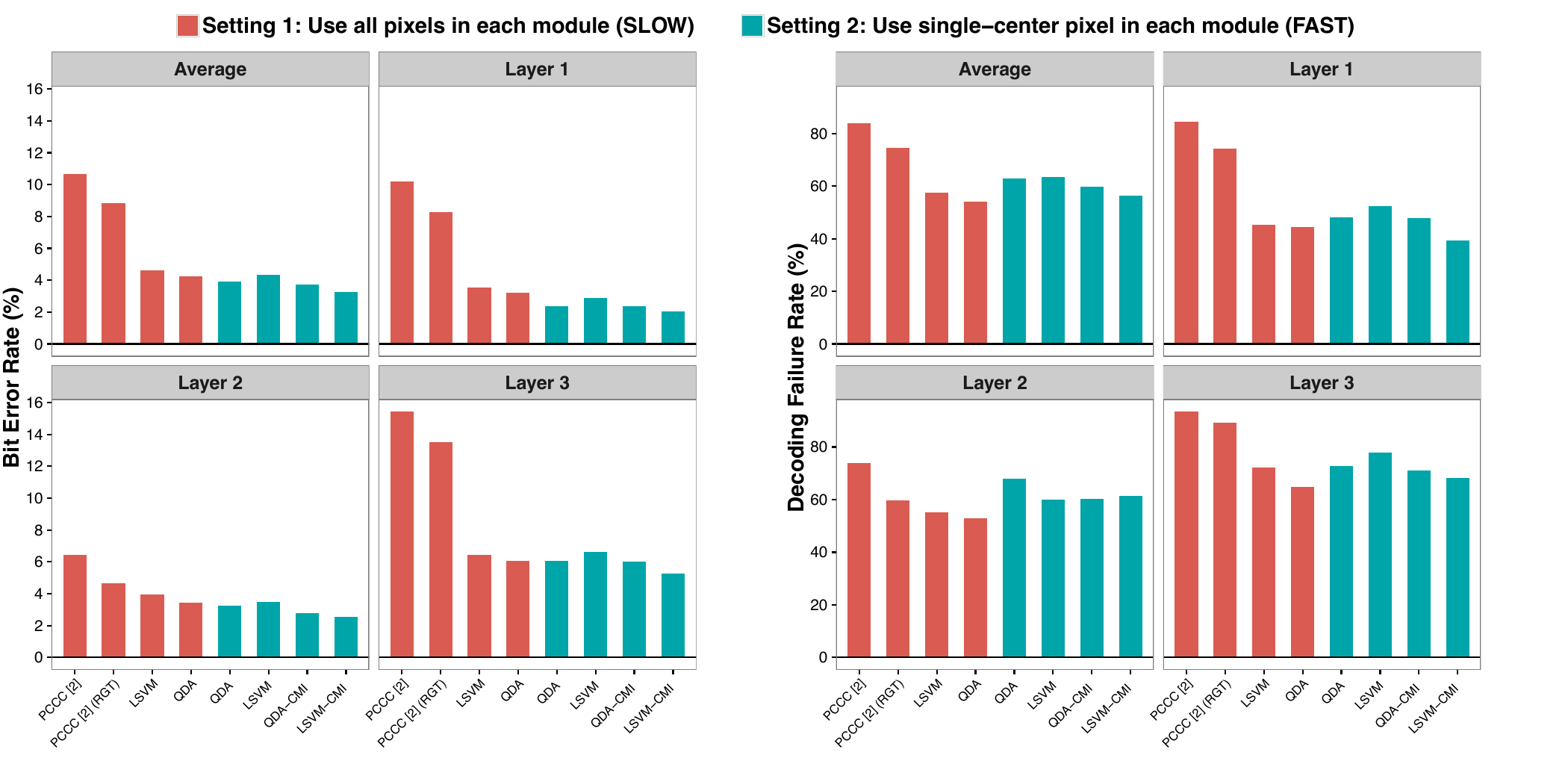}
\caption{Bit error rate (left) and decoding failure rate (right) of different color recovery methods on CUHK-CQRC. \textbf{Setting 1}: four color recovery methods (PCCC \cite{blasinski2013per}, PCCC with RGT, QDA and LSVM) are performed on every pixel of a captured image. \textbf{Setting 2}: four color recovery methods (QDA, LSVM, QDA-CMI and LSVM-CMI) are performed only on the center pixel of each color module.}
\label{fig:rst_DSR_BER}
\end{figure*}

\subsection{Implementation Details}
Although the QR codes in CUHK-CQRC are generated using the proposed encoder (Section \ref{sec:HiQ}), the color QR codes are also compatible with the PCCC decoder. Both HiQ and PCCC are implemented using the decoder part of a popular monochrome QR code implementation, ZXing codebase \cite{mackintosh2012zxing}. As suggested in Table \ref{table:layeredscheme}, LSVM with different kernels has similar performance in the first two layers, so in our implementation of LSVM and LSVM-CMI, we use linear kernel in the first two layers and polynomial kernel of degree three in the third layer to reduce latency.

Since the parameters of HiQ are learned offline, we train the color classifier of HiQ using data sampled from CUHK-CQRC prior to conducting experiments. We select 65 images of color QR codes which cover different lighting conditions, phone models, print resolutions and formats (i.e., photo and preview) for training and use the rest for testing. To collect the color data from the HiQ code images, we use a human-assisted labeling approach. To be more specific, given a captured image of a HiQ code, instead of manually labeling the color pixel by pixel, we only manually input the positions of the markers of the HiQ code and apply the existing geometric transformation algorithm to sample the color data from each square module of which we actually know the ground-truth color. In this way, we can substantially cut down the manual labeling effort while managing to collect over 0.6 million of labeled color-data modules under a wide range of real-world operating/ lighting conditions. Such a rich set of labeled data plays an important role in boosting color classification performance for our learning-based decoding algorithm.

\subsection{Comparing Different Color Classifiers}\label{sec:classifierCompare}
In this section, 
we evaluate the color recovery performance of different machine learning techniques, including LSVM (without incorporating CMI cancellation), one-vs-all SVM, QDA \cite{srivastava2007bayesian} and decision forests \cite{criminisi2011decision} on real-operating color data (one million pixel samples). Table \ref{table:layeredscheme} presents the results. Linear, polynomial (of degree 2 and 3, denoted as Poly-2 and Poly-3, respectively) and RBF kernels are tried in SVM implementation. For random forests, we use depth-9 trees and train 100 of them by using 5 random splits when training each weak learner.
According to Table \ref{table:layeredscheme}, LSVM with RBF kernel appears to be the best choice for our application considering accuracy, but using kernel techniques in SVM on a mobile application is too time-consuming. Alternatively, one can either approximate the RBF kernel via explicit feature mapping \cite{rahimi2007random}\cite{vedaldi2012efficient}, or map the original feature into a slightly higher dimensional space using an as-good kernel such as a low-degree polynomial kernel \cite{chang2010training}. For decoding quickness, we choose the latter in our implementation. Considering speed and accuracy, QDA and LSVM (Poly-3) are more favorable than others. Between QDA and LSVM, QDA is of higher accuracy while LSVM has lower processing latency.

{\color{black}
\subsection{Analysis of Layer-Density Trade-Off}\label{sec:layer_density_tradeoff}
As we discussed in Section \ref{sec:HiQ}, there exists a trade-off between layer and density, namely, given the same amount of user data, a HiQ code with more layers has less module density but it also has more difficulties in color recovery. In this section, we study the layer-density trade-off by comparing the performance of 2-layer and 3-layer HiQ codes.

We print 2-layer and 3-layer HiQ codes of 2900, 4500, 5800 and 8900 bytes of user data and scan these HiQ codes using Nexus 5 under fluorescent lighting. To quantify the performance, we use \textit{limit size} (i.e., the smallest printout size of the QR code that can be decoded) and \textit{prediction per frame (PPF)} as metrics. We use PPF to measure the decoding latency of the HiQ codes instead of scanning time because PPF does not vary with the printout size. In this experiment, we use QDA as the color classifier for both 2-layer (4 colors) and 3-layer (8 colors) HiQ codes. Table \ref{table:layerdensitytradeoff} lists the experimental results. It is shown that given the same content sizes, 2-layer and 3-layer color QR codes have similar limit sizes, and decoding a 2-layer color QR code consumes less time in each frame. Besides, it is also shown that given the same QR code dimension, 2-layer HiQ codes can be denser than 3-layer ones (smaller limit size), which also indicates the difficulties brought by adding one extra layer.
}
\subsection{Evaluation of HiQ on CUHK-CQRC}\label{sec:cqrc_eva}
In this section, we evaluate the performance of HiQ by comparing it with the baseline method, PCCC from \cite{blasinski2013per}, using CUHK-CQRC. Since PCCC performs color recovery on each pixel of a captured image before applying local binarization (Setting 1), while in the proposed QDA-CMI and LSVM-CMI we only perform color recovery on the center pixel of each color module (Setting 2). Performing color recovery on every pixel helps decoding because binarization can benefit from neighboring pixels, but it is prohibitively time-consuming for practical consideration as a captured image usually consists of more than one million pixels. For fair comparison, we conduct two groups of experiments under the above two settings. In Setting 1, we compare PCCC with HiQ which uses QDA and LSVM (Poly-3) as the color classifier to show that our framework actually beats the baseline even without adopting CMI cancellation techniques. In Setting 2, by comparing QDA-CMI and LSVM-CMI with QDA and LSVM, we show the superiority of the proposed QDA-CMI and LSVM-CMI. 

\begin{table*}[t]
\caption{Scanning performance of different color recovery algorithms using iPhone 6s Plus}
\begin{center}
\begin{tabular}{K{1.8cm} |p{2.8cm} K{1.8cm} K{1.8cm} K{1.8cm} K{1.8cm}} 
\hline
\multicolumn{2}{c}{\textbf{}} & \textbf{QDA} & \textbf{LSVM} & \textbf{QDA-CMI} & \textbf{LSVM-CMI} \\
\hline
\hline
\multirow{3}{*}{\shortstack{\textbf{3819 bytes} \\ $\mathbf{35\times35 \,{mm}^2}$}} 
& Number of frames & 1.44 & 1.86 & 1.31 & 1.13\\ \cline{2-6}
& Overall latency (ms) & 375.06 & 372.83 & 372.49 & 264.29\\ \cline{2-6}
& Time per frame (ms) & 260.91 & 200.22 & 283.41 & 234.09\\
\hline
\hline
\multirow{3}{*}{\shortstack{\textbf{5196 bytes} \\ $\mathbf{40\times40\,{mm}^2}$}}
& Number of frames & 1.40 & 1.44 & 1.33 & 1.33\\ \cline{2-6}
& Overall latency (ms) & 435.32 & 308.36 & 479.33 & 365.16\\ \cline{2-6}
& Time per frame (ms) & 310.94 & 214.14 & 359.50 & 275.47\\
\hline
\hline
\multirow{3}{*}{\shortstack{\textbf{6909 bytes} \\ $\mathbf{50\times50\,{mm}^2}$}} 
& Number of frames & 3.65 & 5.51 & 1.61 & 1.24\\ \cline{2-6}
& Overall latency (ms) & 1323.77 & 1295.17 & 670.13 & 401.24\\ \cline{2-6}
& Time per frame (ms) & 362.30 & 234.75 & 415.48 & 323.81\\
\hline
\hline
\multirow{3}{*}{\shortstack{\textbf{8097 bytes} \\ $\mathbf{38\times38\,{mm}^2}$}} 
& Number of frames & - & - & - & 4.80\\ \cline{2-6}
& Overall latency (ms) & - & - & - & 2146.23\\ \cline{2-6}
& Time per frame (ms) & - & - & - & 447.13\\
\hline
\multicolumn{6}{l}{\footnotesize{\enquote{-} means cannot be successfully decoded.}}
\end{tabular}
\end{center}
\label{table:qdalsvm}
\end{table*}

The results presented in Fig. \ref{fig:rst_DSR_BER} show that, in Setting 1, HiQ with QDA reduces BER from 10.7\% to 4.3\% and DFR from 84\% to 54\% (the decoding success rate is increased by 188\%) compared with PCCC. Moreover, we also apply robust geometric transformation (RGT) on PCCC, denoted as \textit{PCCC (RGT)}. RGT is shown to reduce the DFR and BER of PCCC by 12\% and 18\%, respectively. As for QDA and LSVM, they are comparable in decoding performance as also indicated in Table \ref{table:layeredscheme}. The results also indicate that, across all of the 3 schemes under test, the third layer (yellow channel in PCCC) always yields the worst performance. This is because the color classifier has poor performance in distinguishing between yellow and white which are encoded as $001$ and $000$ respectively in the codebook (see Fig. \ref{fig:cqr_encode_decode}), especially under strong light. Likewise, the classifier performs poorly in distinguishing blue ($110$) and black ($111$) under dim light. The combined effect is that the third layer often cannot be decoded reliably during poor lighting conditions. Fortunately, it is possible to apply a higher error correction level on the third layer to compensate for the higher classification error rate, which will be investigated in Section \ref{sec:exp_mobile}.

In Setting 2, 
both QDA-CMI and LSVM-CMI models outperform their base models (QDA and LSVM) in both DFR and BER. In particular, LSVM-CMI reduces the BER of LSVM by 16.8\% while QDA-CMI reduces the BER of QDA by 6.8\%. Compared with LSVM-CMI, the performance of QDA-CMI is inferior and QDA-CMI is shown to have less significant improvement over its base model. This is probably due to the fact that the objective function of QDA-CMI (see Eq. \eqref{eq:mle}) is in general non-convex while the objective of LSVM-CMI is convex. Consequently, the optimization of QDA-CMI (Algorithm \ref{algo:QDACMI}) is likely to be stuck at a local optimum and yields a suboptimal solution. Yet another limitation of QDA-CMI is that the data points lying along the edges of the Gaussian distribution unfavorably affect the optimization. In other words, a small number of data points can significantly change the value of the objective function while having negligible effects on reducing the prediction error.

Although the overall decoding failure rate of our method (over 50\%) may look high, if one frame fails, the smartphone can instantly capture a new image and start a new round of decoding until the QR code is successfully decoded. Therefore, besides accuracy, processing latency also serves as a key aspect in measuring the practicability of one approach. In the following, we will study the performance of different methods in real-world practice, considering both accuracy and processing latency.

\begin{figure*}[t]
\centering
\subfigure[Experimental Results of LSVM using iPhone 6 Plus.]{
\label{fig:ip6p}
\includegraphics[height=.205\textwidth]{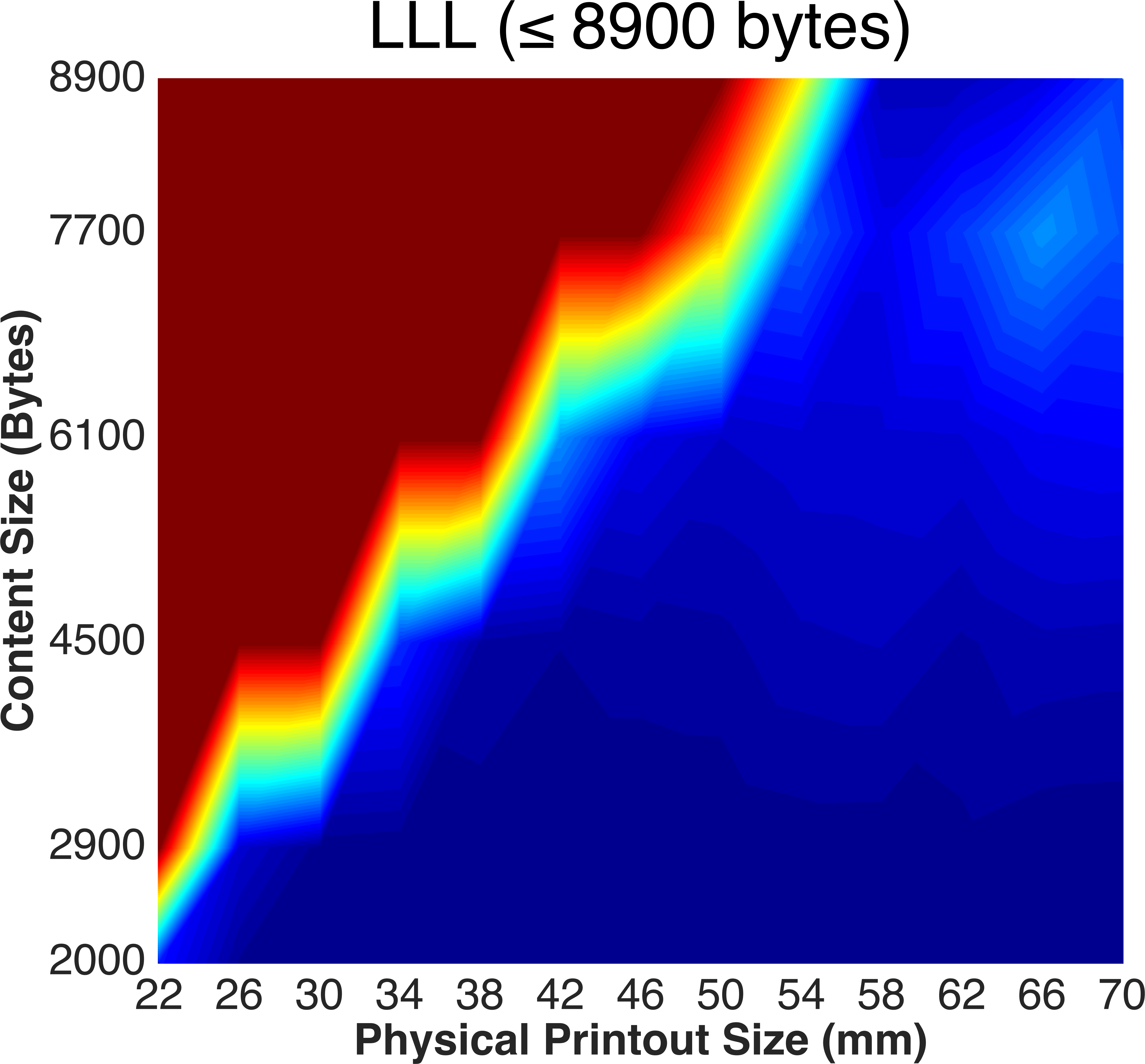}
\includegraphics[height=.205\textwidth]{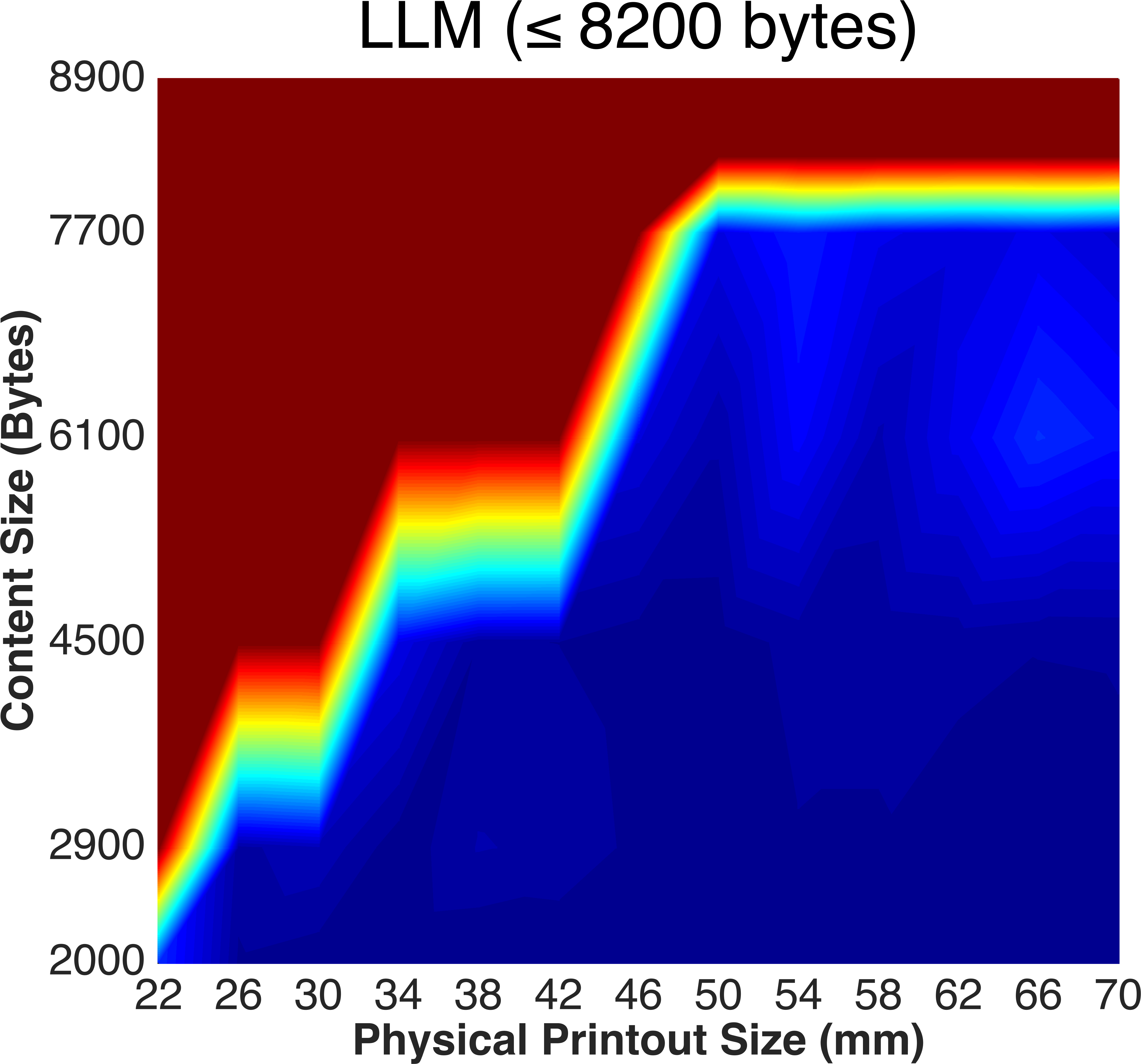}
\includegraphics[height=.205\textwidth]{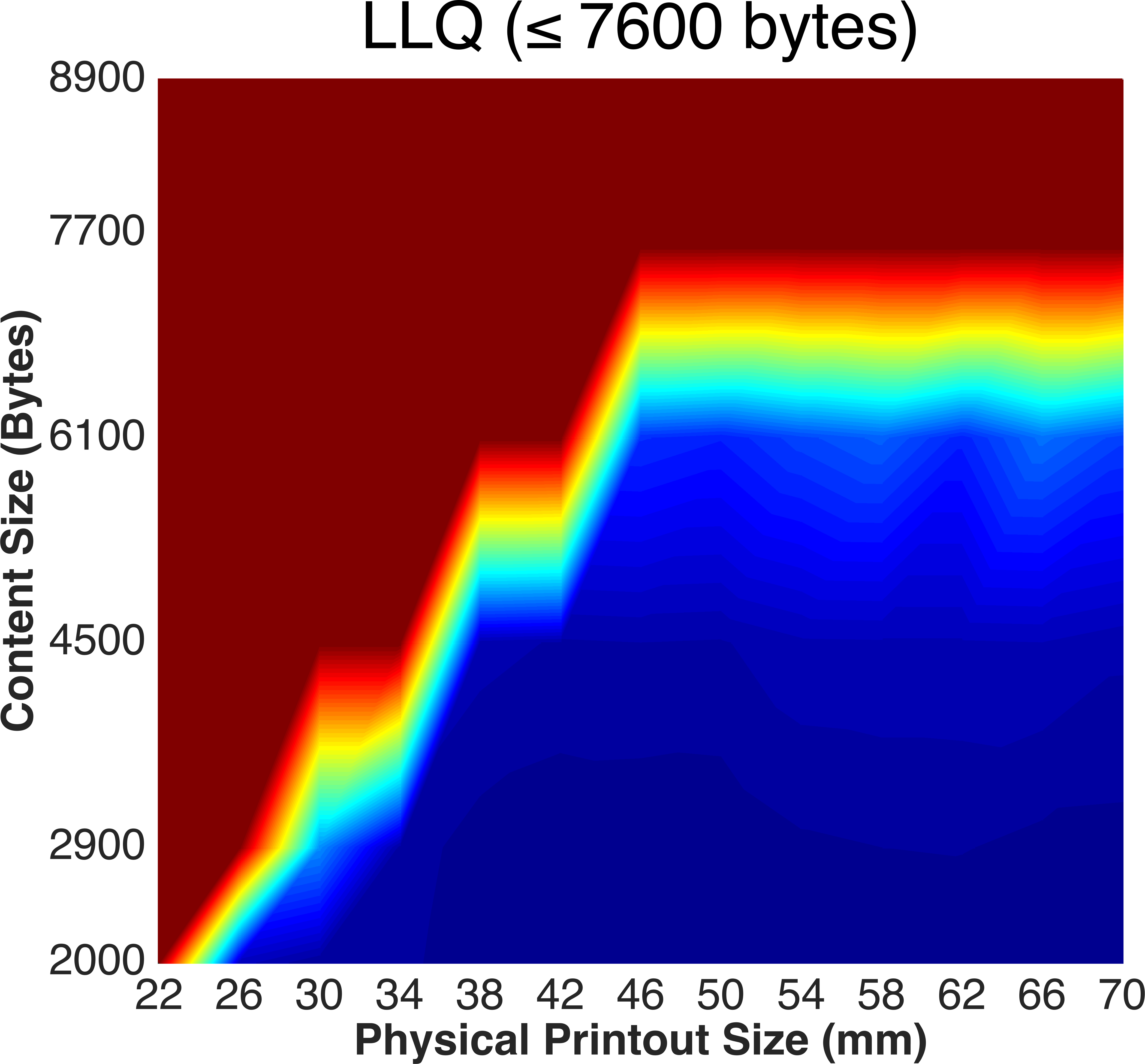}
\includegraphics[height=.205\textwidth]{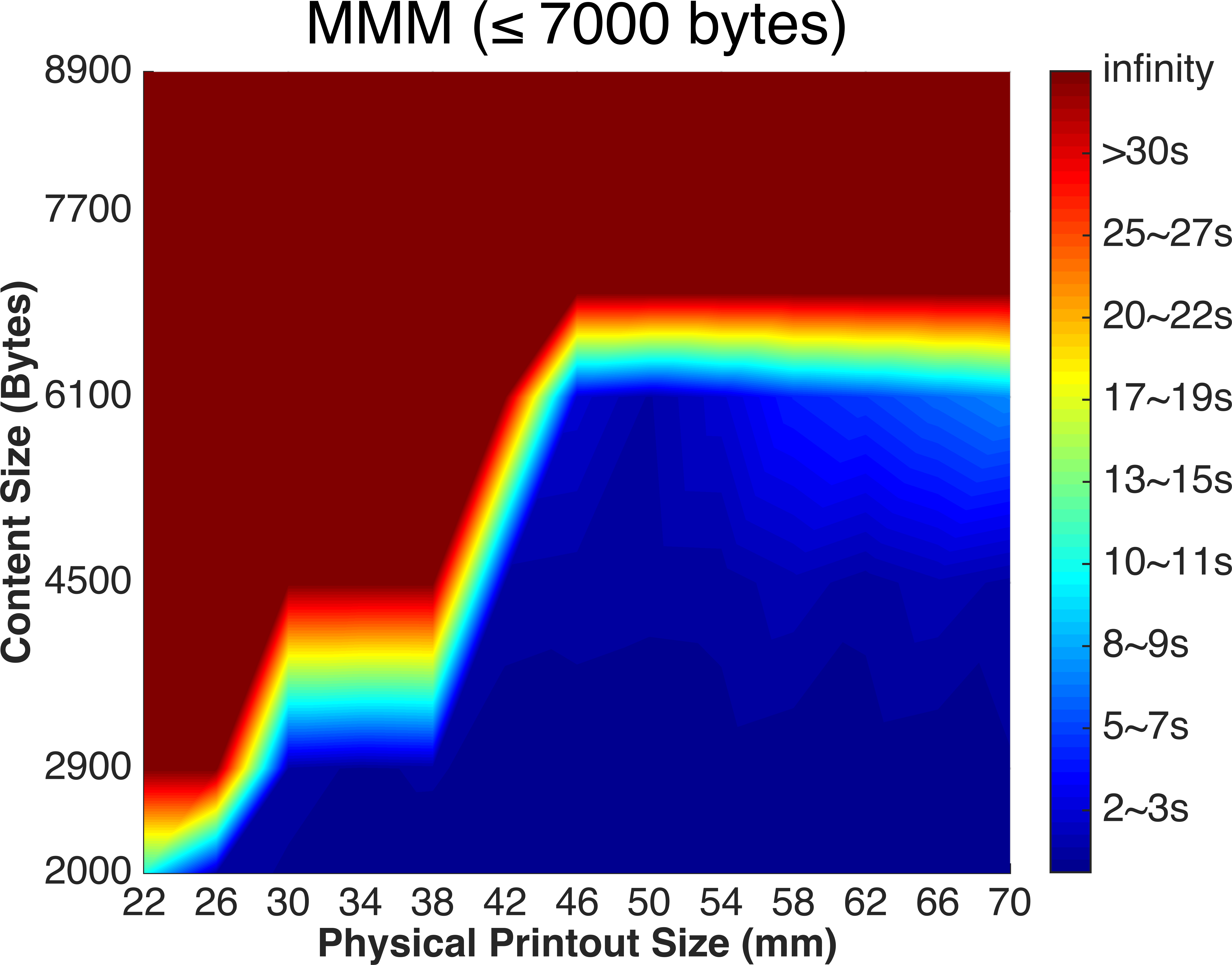}}


\subfigure[Experimental Results of LSVM-CMI using iPhone 6 Plus.]{
\label{fig:ip6p_cmi}
\includegraphics[height=.193\textwidth]{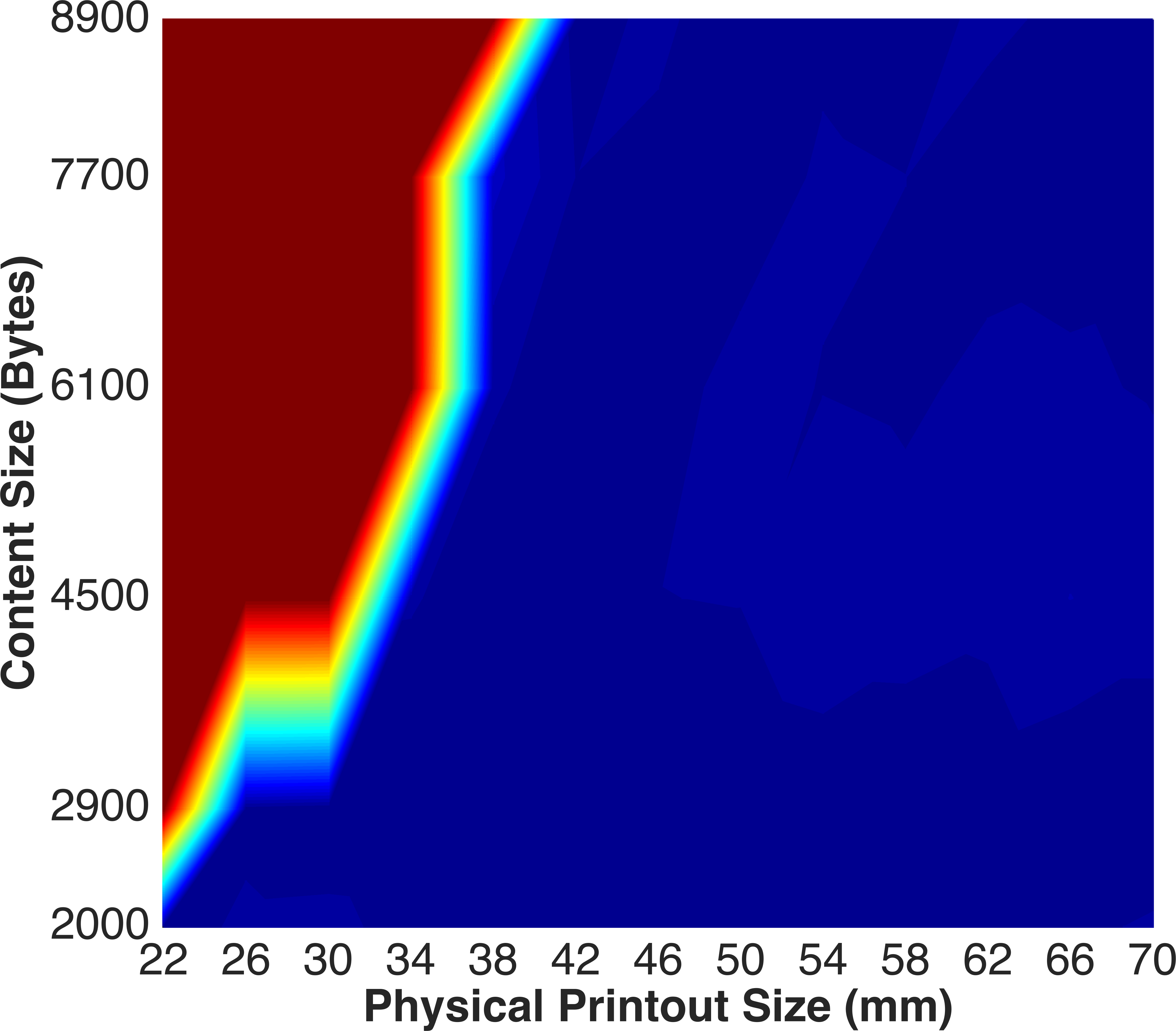}
\includegraphics[height=.193\textwidth]{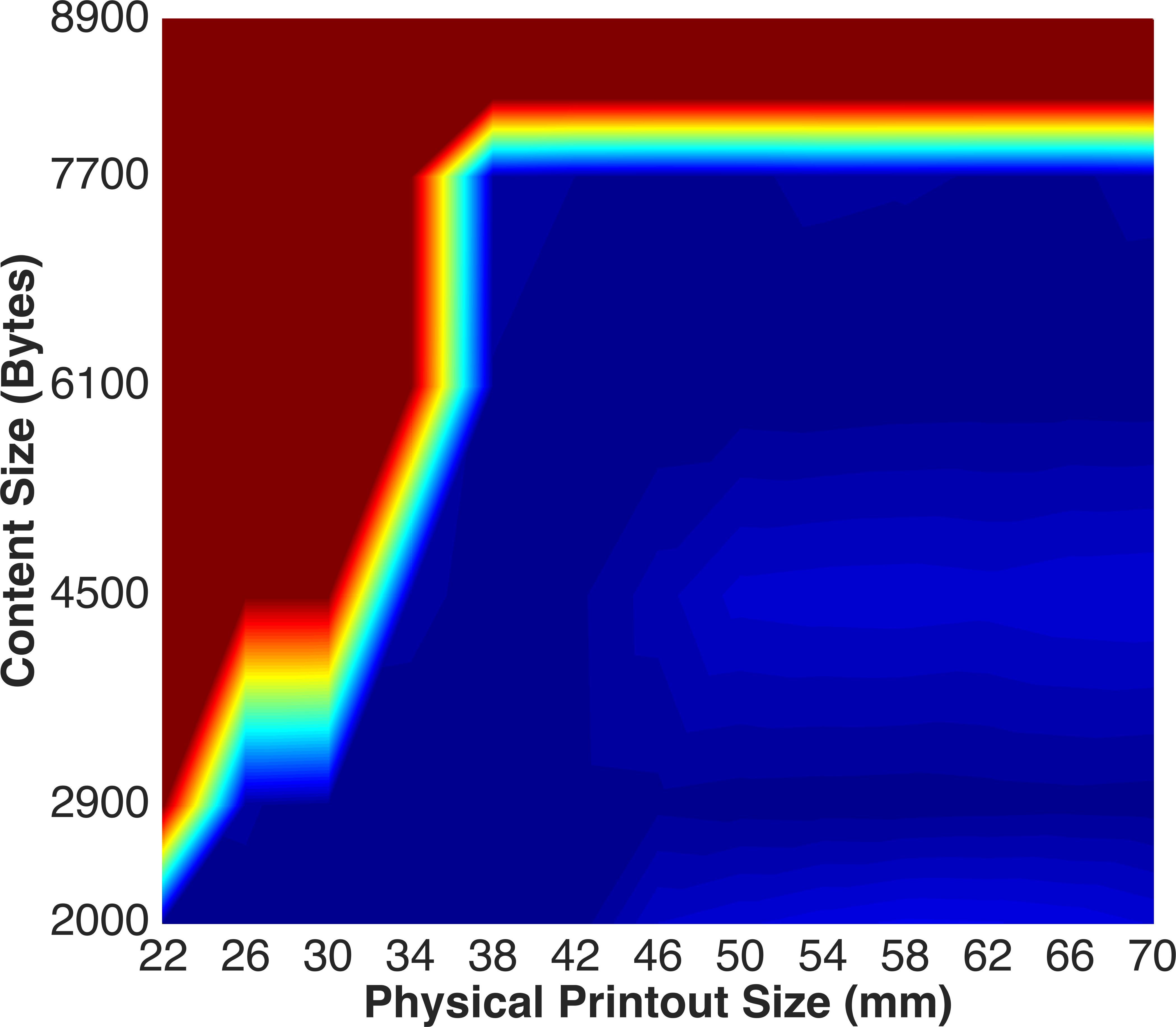}
\includegraphics[height=.193\textwidth]{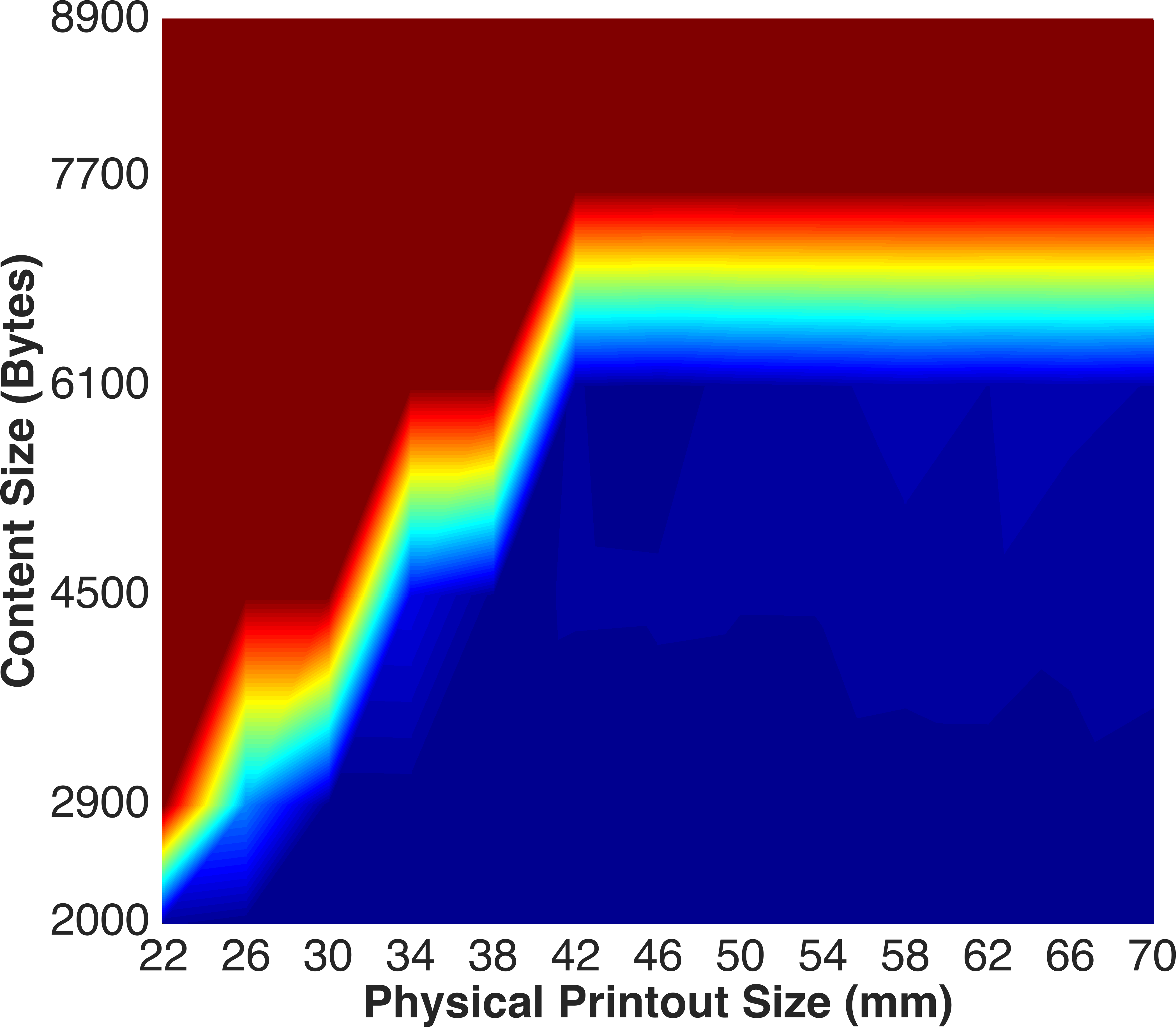}
\includegraphics[height=.193\textwidth]{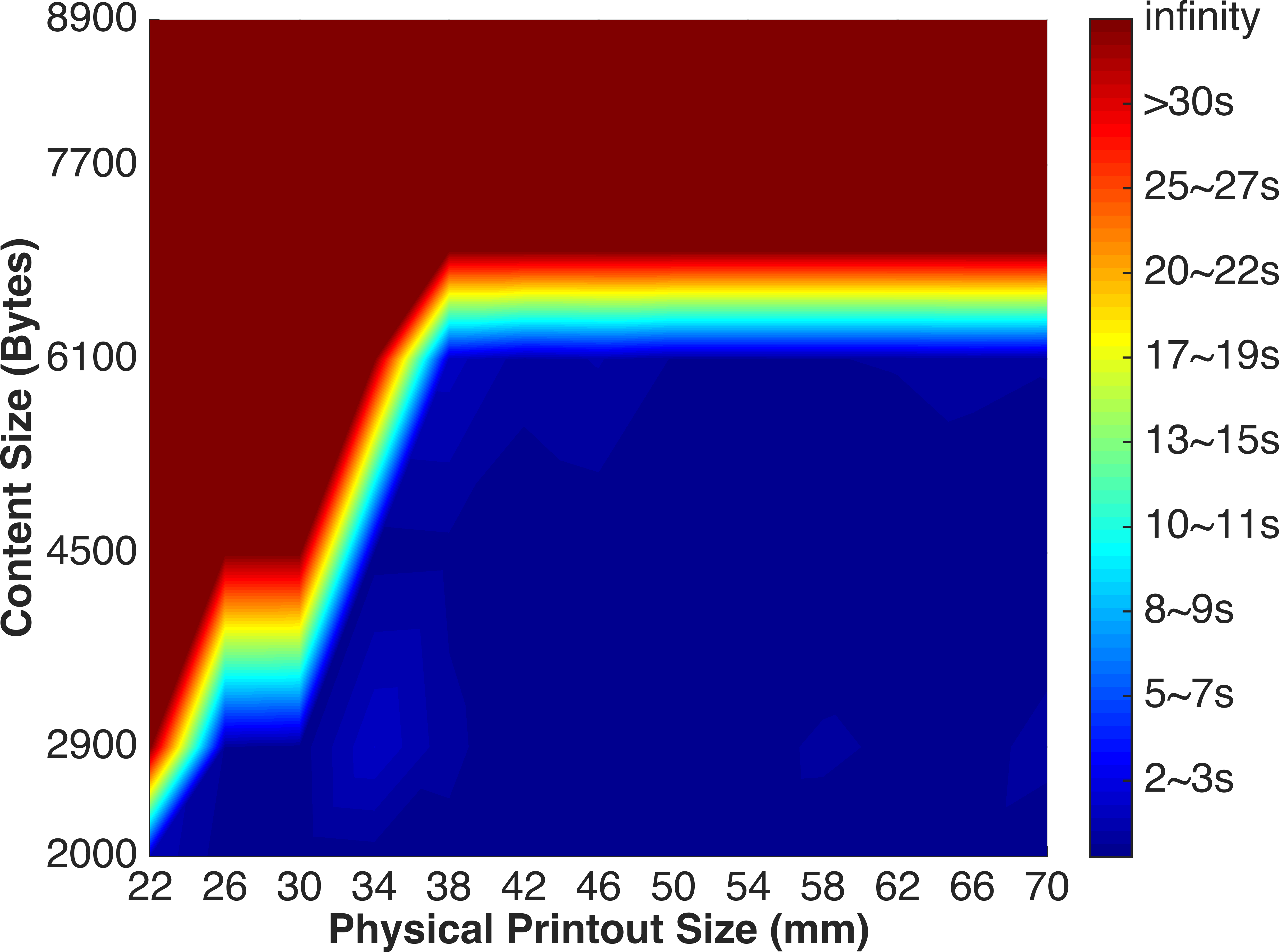}}

\caption{The 90th percentile of the scanning time of 30 trials (in ascending order). From left to right, the four columns use HiQ codes of different error correction levels---LLL ($\le 8900$ bytes), LLM ($\le 8200$ bytes), LLQ ($\le 7600$ bytes), MMM ($\le 7000$ bytes), respectively. The scanning time of the HiQ codes beyond its maximal capacity is set to infinity.
}
\label{fig:heatmap}
\end{figure*}

\subsection{Evaluation of HiQ on Mobile Devices}\label{sec:exp_mobile}
In this section, we demonstrate the effectiveness of HiQ using off-the-shelf smartphones, which include Google Nexus 5, iPhone 6 Plus and iPhone 7 Plus. We investigate several interesting questions: 
1) Compared with QDA, LSVM is superior in speed but inferior in accuracy (see Table \ref{table:layeredscheme}), so how does their performance differ in real-world mobile applications?
2) How do different color recovery methods proposed in this paper---QDA, LSVM, QDA-CMI and LSVM-CMI---perform in real-world scenario? 
3) What is the impact of different error correction levels on the performance of HiQ? and
4) How does the decoding performance vary with respect to the physical printout size of HiQ codes.
As evaluation metrics, we collect the scanning time of 30 successful scans (i.e., trials where the HiQ code is successfully decoded) for each printout HiQ code and the experiments are conducted in an indoor office environment. A trial is unsuccessful if the code cannot be decoded within 60 seconds, and a HiQ code is regarded as \textit{undecodable} if three consecutive unsuccessful trials occur.


Table \ref{table:qdalsvm} lists the experimental results of different color recovery algorithms in real-world mobile applications by using iPhone 6s Plus as representative. In this experiment, we select several challenging HiQ codes samples of low level of error correction in all layers, different data capacities (3819 bytes, 5196 bytes, 6909 bytes and 8097 bytes) and different printout sizes ($35\times35\,\rm{mm}^2, 40\times40\,\rm{mm}^2$, $50\times50\,\rm{mm}^2$ and $38\times38\,\rm{mm}^2$). The results show that  LSVM takes less time to successfully decode a HiQ code compared to QDA (although LSVM takes more frames to complete decoding a HiQ code, it consumes less time to process a frame). By incorporating CMI cancellation, QDA-CMI and LSVM-CMI reduce the overall latency and the increase of computation time for processing a frame is negligible. More importantly, LSVM-CMI achieves the best performance among all methods. In particular, LSVM-CMI has the lowest overall latency and is the only method that can decode the most challenging HiQ code (8097 bytes, $38\times38\,\rm{mm}^2$) with reasonable computational cost.

\begin{table*}[t!]
\centering
\caption{Execution time of basic blocks in the pipeline of HiQ decoder using Nexus 5} 
\begin{tabular}{C{1.3cm}|c|C{01cm}|C{0.9cm}|C{0.9cm}|C{1.12cm}|C{1.05cm}|C{1.08cm}|C{1.1cm}|C{1.25cm}|C{1.05cm}|C{1.2cm}}
\hline
Number of Modules & Type &Data Capacity &YUV-2-RGB & Binarization & Patterns Detection & Transformation & Color Recovery & Randomization & Error Correction & Time per Frame & Number of Frames\\
\hline
\hline
\multirow{2}*{$137\times137$} & B/W & 1732 bytes & NA & \textbf{110ms (34\%)} & \textbf{112ms (34\%)} & 23ms (7\%) & NA & 20ms (6\%) & 41ms (13\%) & 322ms &  3.0  \\\cline{2-12}

                            & Color & 5196 bytes & \textbf{400ms (27\%)} & 204ms (14\%) & 153ms (10\%) & 14ms (1\%) & \textbf{500ms (34\%)} & 45ms (3\%) & 150ms (10\%) & 1466ms & 4.5 \\
\hline
\multirow{2}*{$157\times157$} & B/W & 2303 bytes & NA & \textbf{104ms (27\%)} & \textbf{123ms (32\%)} & 37ms (10\%) & NA & 38ms (10\%) & 60ms (16\%) & 380ms &  4.3  \\\cline{2-12}

                            & Color & 6909 bytes & \textbf{386ms (24\%)} & 200ms (12\%) & 150ms (9\%) & 20ms (1\%) & \textbf{650ms (40\%)} & 69ms (4\%) & 160ms (10\%) & 1635ms & 6.7 \\
\hline
\multirow{2}*{$177\times177$} & B/W & 2953 bytes & NA & \textbf{112ms (25\%)} & \textbf{138ms (30\%)} & 37ms (8\%) & NA & 50ms (11\%) & 97ms (21\%) & 455ms & 5.6   \\\cline{2-12}

                            & Color & 8859 bytes & \textbf{400ms (20\%)} & 193ms (10\%) & 213ms (11\%) & 25ms (2\%) & \textbf{881ms (44\%)} & 111ms (5\%) & 200ms (10\%) & 2023ms & 9.0 \\
\hline
 \end{tabular}
 \label{PipelineAna}
 \end{table*}

In the following, we evaluate the performance of LSVM and LSVM-CMI and show the effectiveness of HiQ and superiority of LSVM-CMI over LSVM in real-world practice. We conduct experiments using 3-layer HiQ codes with different error correction levels and different content sizes using iPhone 6 Plus and iPhone 7 Plus. More specifically, we choose six different content sizes, 2000 bytes, 2900 bytes, 4500 bytes, 6100 bytes, 7700 bytes and 8900 bytes (approximately). For each content size, we generate color QR codes using four different levels of error correction which are denoted by 4 triplets, LLL, LLM, LLQ and MMM. {\color{black}Note that the data capacity of a QR code will be reduced if higher error correction level is used. Therefore, we cannot apply the four different error corrections for all content sizes. For instance, with a content size of 8900 bytes, we can only use LLL, and for content size of 7700 bytes, only LLL and LLM are feasible.} Each symbol (L, M and Q) of the triplet represents different level of error correction (low, median and quartile
\footnote{Low, median and quartile level of error correction can correct up to 7\%, 15\% and 25\% codewords, respectively.}
) applied for the corresponding layer.
We try different error correction levels in the third layer as it has shown to be the most error-prone layer (see Table \ref{table:layeredscheme}). 
Each generated HiQ code is printed in different printout sizes ranging from $22\,\rm{mm}$ to $70\,\rm{mm}$.

The results presented in Fig. \ref{fig:heatmap} show that HiQ decoder can, in most of the cases, successfully decode the HiQ code within 5 seconds with small variance (see the supplement for more detailed results). Fig. \ref{fig:heatmap} also conveniently shows the smallest printout sizes of the color QR codes with different content sizes that can be decoded in a reliable and rapid manner. 
Comparing the performance of different error correction levels, we can see that, in most cases, LLL and LLM outperform LLQ and MMM in terms of the smallest decodable printout size given the same content size. This suggests that it is not helpful to apply error correction level that is higher than $M$ since higher level of error correction not only increases the error-tolerant ability, but also increases the data density by adding more redundancies.
More importantly, the comparison between the first two rows of Fig. \ref{fig:heatmap} demonstrates the effectiveness of the proposed CMI cancellation. In particular, LSVM-CMI not only outperforms LSVM in overall scanning latency, but also in minimum decodable printout size of the HiQ codes, e.g, with LSVM-CMI reduces the minimum printout size of 7700-byte HiQ code by 24\% (from $50\,\rm{mm}$ to $38\,\rm{mm}$). {\color{black}We also conducted experiments using iPhone 7 Plus to evaluate our approach across different devices (see detailed results in the supplement). We find that HiQ can achieve similar performance for iPhone 6 Plus and iPhone 7 Plus, though iPhone 7 Plus is faster because of a better processor.}

Lastly, 
we evaluate the performance of spatial randomization for HiQ (see Section \ref{sec:randaccu}) on a mobile device. Using Nexus 5 as a representative phone model, we compare the decoding performance on randomized and non-randomized HiQ codes of different content sizes and printout sizes. Specifically, we choose 6100-byte and 7700-byte HiQ code samples to do the evaluation.
Fig. \ref{fig:acmcompare} presents the block accumulation behavior when decoding HiQ codes with and without randomization. One can observe that the randomized samples have higher starting point of successful blocks percentage and higher accumulation speed, while the decoding of the original samples easily fails. We also found that randomization also improves the decodability of HiQ codes, especially high-capacity ones. For instance, for 6600-bytes HiQ codes, the use of randomization pushes the minimum decodable printout size from $46\rm{mm}$ to $42\rm{mm}$, and cuts down the average scanning latency by over 50\% given certain printout sizes.

\begin{figure}[t]
\centering
\includegraphics[width=0.48\textwidth]{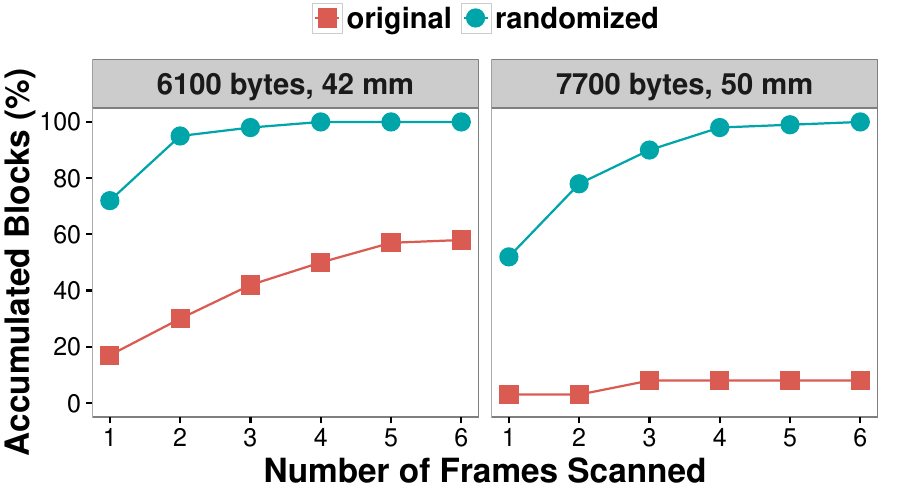}
\caption{Comparison of scanning performance with respect to block accumulation and randomization.}
\label{fig:acmcompare}
\end{figure}

{\color{black}
\subsection{Evaluation of CMI across Printers}
One natural question one may raise is: does CMI models derived from the output of one printer generalize to other printers? In this section, we demonstrate the effectiveness of our CMI model by testing LSVM and LSVM-CMI trained for Ricoh MP C5501A over HiQ codes printed by two different printers: HP DeskJet 2130 (low-end inkjet printer) and Ricoh MP C6004 (high-end laser printer). In particular, we print HiQ codes of three different capacity--2000 bytes, 2919 bytes, 4498 bytes and 6106 bytes--in different printout sizes ranging from $22\rm{mm}$ to $54\rm{mm}$ using the aforementioned two printers and use iPhone 6s Plus as the scanning device.

Fig. \ref{fig:printer} presents the experimental results of LSVM and LSVM-CMI. For both printers, LSVM-CMI is more efficient than LSVM in terms of overall scanning time and LSVM-CMI can decode denser HiQ codes than LSVM. For example, for a 4395-byte HiQ code, LSVM-CMI can decode up $30\rm{mm}$, while the smallest decodable printout size for LSVM is $3.2\rm{mm}$. And for 6105-byte HiQ code, LSVM-CMI can decode it with printout size $54\rm{mm}$ while LSVM cannot. More importantly, this also implies the color mixing parameters of CMI model learned for one printer can be directly applied to other printers.
However, for HP printer, both LSVM and LSVM-CMI suffers huge drop in scanning performance compared with Ricoh printer. Both LSVM and LSVM-CMI can only decode several of the printed HiQ codes (see Fig. \ref{fig:hp}). Two reasons lead to this: 1) As a lower-end color printer, HP DeskJet 2130 does not produce as good printing quality as Ricoh MP C6004; 2) The color tone of the colorants used by HP printer and Ricoh printer differs significantly, while our models are trained using data from Ricoh MP C5501A which uses similar ink as Ricoh MP C6004. Fortunately, we can address this problem by training models over color data from different types of color ink, and we leave it for  future investigation.
}

\begin{figure}[t]
\centering
\subfigure[Results of Ricoh MP C6004.]{\label{fig:c6004}
  \includegraphics[width=0.48\textwidth]{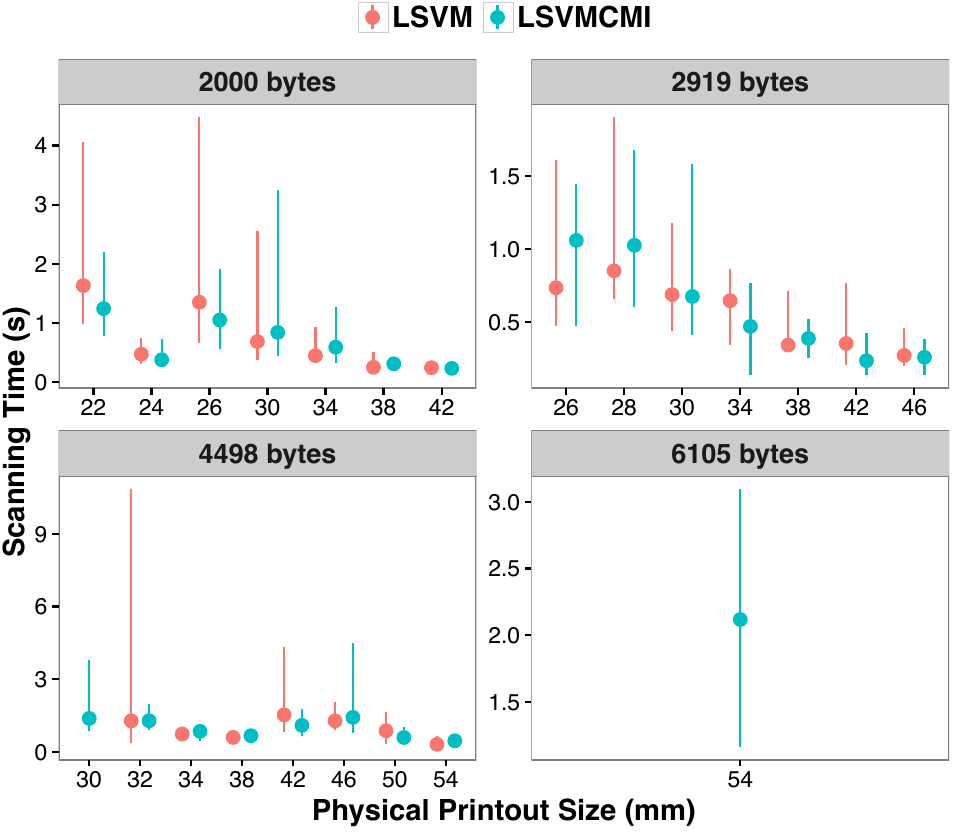}
\vspace{-1cm}
}
\subfigure[Results of HP DeskJet 2130.]{\label{fig:hp}
\includegraphics[width=.48\textwidth]{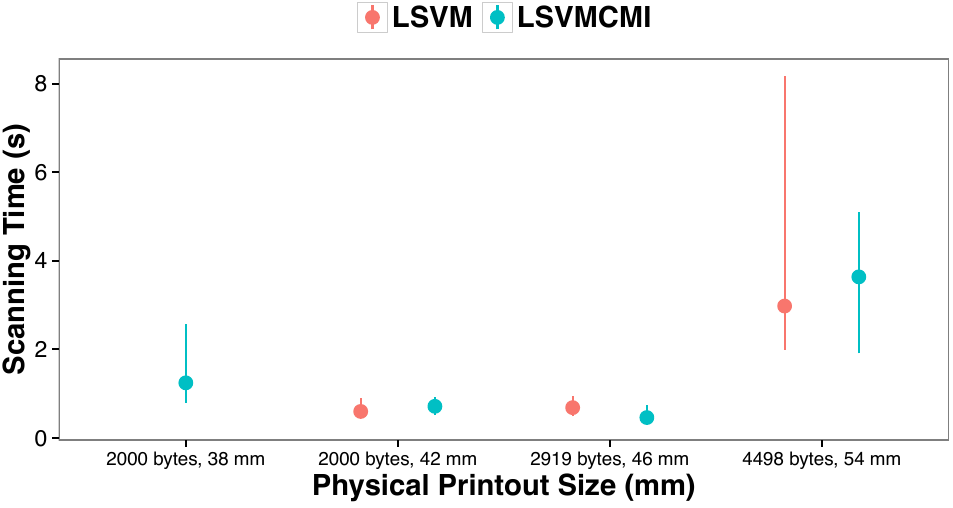}}
\caption{{\color{black}Comparison of the scanning performance of LSVM and LSVM-CMI for Ricoh MP C6004 and HP DeskJet 2130.}} 
\label{fig:printer}
\end{figure}
\subsection{Pipeline Analysis}
We examine the execution time of sequential flows (basic blocks) in the pipeline of our proposed HiQ framework by using monochrome and 3-layer HiQ codes. Specifically, we choose three different versions of QR codes which consist of $137\times137$, $157\times157$ and $177\times177$ modules, respectively. For each HiQ code, we evaluate the execution time of each block by averaging 10 scans using Google Nexus 5.

Observed from Table \ref{PipelineAna}, the most time-consuming parts for HiQ is color recovery and the YUV-to-RGB conversion, taking up around 40\% and 20\%, respectively. Note that YUV-to-RGB conversion is a necessary step for the implementation on Android to transfer the captured image format from YUV to RGB, but not for iOS. Besides, the randomization part only takes up no more than 11\% (120 ms) of the scanning time for both single-layer and 3-layer HiQ codes, which is acceptable in practice.

\section{Conclusion}\label{sec:concl}
In this paper, we have proposed two methods that jointly model different types of chromatic distortion (cross-channel color interference and illumination variation) together with newly discovered chromatic distortion, cross-module color interference, for high-density color QR codes. A robust geometric transformation method is developed to address the challenge of geometric distortion.
Besides, we have presented a framework for high-capacity color QR codes, HiQ, which enables users and developers to create generalized QR codes with flexible and broader range of choices of data capacity, error correction and color, etc. To evaluate the proposed approach, we have collected the first large-scale color QR code dataset, CUHK-CQRC. Experimental results have shown substantial advantages of the HiQ over the baseline approach. Our implementation of HiQ on both Android and iOS and evaluation using off-the-shelf smartphones have demonstrated its usability and effectiveness in real-world practice.
In the future, as opposed to current design where error correction is performed layer by layer, a new mechanism will be developed to share correction capacity across layers by constructing error correction codes and performing correction for all layers as a whole, by which we think the robustness of our color QR code system will be further improved.

\bibliographystyle{IEEEtran}
\bibliography{IEEEabrv,paper}

\begin{thebibliography}{10}
\providecommand{\url}[1]{#1}
\csname url@samestyle\endcsname
\providecommand{\newblock}{\relax}
\providecommand{\bibinfo}[2]{#2}
\providecommand{\BIBentrySTDinterwordspacing}{\spaceskip=0pt\relax}
\providecommand{\BIBentryALTinterwordstretchfactor}{4}
\providecommand{\BIBentryALTinterwordspacing}{\spaceskip=\fontdimen2\font plus
\BIBentryALTinterwordstretchfactor\fontdimen3\font minus
  \fontdimen4\font\relax}
\providecommand{\BIBforeignlanguage}[2]{{%
\expandafter\ifx\csname l@#1\endcsname\relax
\typeout{** WARNING: IEEEtran.bst: No hyphenation pattern has been}%
\typeout{** loaded for the language `#1'. Using the pattern for}%
\typeout{** the default language instead.}%
\else
\language=\csname l@#1\endcsname
\fi
#2}}
\providecommand{\BIBdecl}{\relax}
\BIBdecl

\bibitem{yang2016towards}
Z.~Yang, Z.~Cheng, C.~C. Loy, W.~C. Lau, C.~M. Li, and G.~Li, ``Towards robust
  color recovery for high-capacity color qr codes,'' in \emph{Proc. IEEE Int.
  Conf. Image Process. (ICIP)}, Sept. 2016, pp. 2866--2870.

\bibitem{blasinski2013per}
H.~Blasinski, O.~Bulan, and G.~Sharma, ``Per-colorant-channel color barcodes
  for mobile applications: An interference cancellation framework,''
  \emph{{IEEE} Trans. Image Process.}, vol.~22, no.~4, pp. 1498--1511, Apr.
  2013.

\bibitem{liu2008recognition}
Y.~Liu, J.~Yang, and M.~Liu, ``Recognition of qr code with mobile phones,'' in
  \emph{Control and Decision Conference, 2008. CCDC 2008. Chinese}.\hskip 1em
  plus 0.5em minus 0.4em\relax IEEE, 2008, pp. 203--206.

\bibitem{soon2008qr}
T.~J. Soon, ``Qr code,'' \emph{Synthesis Journal}, vol. 2008, pp. 59--78, 2008.

\bibitem{li2015authpaper}
C.~M. Li, P.~Hu, and W.~C. Lau, ``Authpaper: Protecting paper-based documents
  and credentials using authenticated {2D} barcodes,'' in \emph{IEEE Int. Conf.
  Commun. (ICC)}, Jun. 2015, pp. 7400--7406.

\bibitem{hartley2003multiple}
R.~Hartley and A.~Zisserman, \emph{Multiple view geometry in Computer
  Vision}.\hskip 1em plus 0.5em minus 0.4em\relax Cambridge university press,
  2003.

\bibitem{gijsenij2012improving}
A.~Gijsenij, T.~Gevers, and J.~Van De~Weijer, ``Improving color constancy by
  photometric edge weighting,'' \emph{{IEEE} Trans. Pattern Anal. Mach.
  Intell.}, vol.~34, no.~5, pp. 918--929, May 2012.

\bibitem{grillo2010high}
A.~Grillo, A.~Lentini, M.~Querini, and G.~F. Italiano, ``High capacity colored
  two dimensional codes,'' in \emph{Proc. IEEE Int. Multiconf. Comput. Sci.
  Inf. Technol.}, Oct. 2010, pp. 709--716.

\bibitem{kato2009novel}
H.~Kato, K.~T. Tan, and D.~Chai, ``Novel colour selection scheme for {2D}
  barcode,'' in \emph{Proc. IEEE Int. Symp. Intell. Signal Process. Commun.
  Syst.}, Jan. 2009, pp. 529--532.

\bibitem{onoda2005hierarchised}
T.~Onoda and K.~Miwa, ``Hierarchised two-dimensional code, creation method
  thereof, and read method thereof,'' \emph{available at Japan Patent Office},
  vol. 213336, 2005.

\bibitem{parikh2008localization}
D.~Parikh and G.~Jancke, ``Localization and segmentation of a {2D} high
  capacity color barcode,'' in \emph{IEEE Workshop Appl. of Comput. Vis.
  (WACV)}, Jan. 2008, pp. 1--6.

\bibitem{querini2013color}
M.~Querini and G.~F. Italiano, ``Color classifiers for {2D} color barcodes,''
  in \emph{Proc. IEEE Fed. Conf. Comput. Sci. and Inf. Syst.}, Sept. 2013, pp.
  611--618.

\bibitem{chen2016picode}
C.~Chen, W.~Huang, B.~Zhou, C.~Liu, and W.~H. Mow, ``Picode: A new
  picture-embedding 2d barcode,'' \emph{{IEEE} Trans. Image Process.}, vol.~25,
  no.~8, pp. 3444--3458, Aug. 2016.

\bibitem{hao2012cobra}
T.~Hao, R.~Zhou, and G.~Xing, ``Cobra: color barcode streaming for smartphone
  systems,'' in \emph{Proc. ACM 10th Int. Conf. Mobile syst., appl., serv.
  (MobiSys)}, Jun. 2012, pp. 85--98.

\bibitem{hu2014strata}
W.~Hu, J.~Mao, Z.~Huang, Y.~Xue, J.~She, K.~Bian, and G.~Shen, ``Strata:
  layered coding for scalable visual communication,'' in \emph{Proc. ACM 20th
  annual Int. Conf. Mobile comput. netw. (MobiCom)}, Sept. 2014, pp. 79--90.

\bibitem{hermans2016focus}
F.~Hermans, L.~McNamara, G.~S{\"o}r{\"o}s, C.~Rohner, T.~Voigt, and E.~Ngai,
  ``Focus: Robust visual codes for everyone,'' in \emph{Proc. ACM 14th Annual
  Int. Conf. Mobile syst., appl., serv. (MobiSys)}, Jun. 2016, pp. 319--332.

\bibitem{bagherinia2011theory}
H.~Bagherinia and R.~Manduchi, ``A theory of color barcodes,'' in \emph{IEEE
  Int. Conf. Comput. Vis. Workshops (ICCV Workshops)}, Nov. 2011, pp. 806--813.

\bibitem{shimizu2011color}
T.~Shimizu, M.~Isami, K.~Terada, W.~Ohyama, and F.~Kimura, ``Color recognition
  by extended color space method for 64-color 2-d barcode.'' in \emph{MVA},
  Jan. 2011, pp. 259--262.

\bibitem{bagherinia2014novel}
H.~Bagherinia and R.~Manduchi, ``A novel approach for color barcode decoding
  using smart phones,'' in \emph{Proc. IEEE Int. Conf. Image Process.}, Oct.
  2014, pp. 2556--2559.

\bibitem{hsu2002comparison}
C.-W. Hsu and C.-J. Lin, ``A comparison of methods for multiclass support
  vector machines,'' \emph{{IEEE} Trans. Neural Netw.}, vol.~13, no.~2, pp.
  415--425, Mar. 2002.

\bibitem{simonot2014between}
L.~Simonot and M.~H{\'e}bert, ``Between additive and subtractive color mixings:
  intermediate mixing models,'' \emph{J. Opt. Soc. Am. A Opt. Image Sci. Vis.},
  vol.~31, no.~1, pp. 58--66, Jan. 2014.

\bibitem{l1normshrinkage}
R.~Tibshirani, M.~Saunders, S.~Rosset, J.~Zhu, and K.~Knight, ``Sparsity and
  smoothness via the fused lasso,'' \emph{Journal of the Royal Statistical
  Society: Series B (Statistical Methodology)}, vol.~67, no.~1, 2005.

\bibitem{l1regularlizationsparsecamera}
A.~Kadambi, R.~Whyte, A.~Bhandari, and L.~Streeter, ``Coded time of flight
  cameras: Sparse deconvolution to address multipath interference and recover
  time profiles,'' in \emph{Proceedings of SIGGRAPH Asia}, 2013.

\bibitem{nonlinear_programming}
D.~Bertsekas, \emph{Nonlinear Programming}.\hskip 1em plus 0.5em minus
  0.4em\relax Athena Scientific, 1999.

\bibitem{cheng2015effective}
D.~Cheng, B.~Price, S.~Cohen, and M.~S. Brown, ``Effective learning-based
  illuminant estimation using simple features,'' in \emph{Proc. IEEE Conf.
  Comput. Vis. and Pattern Recognit. (CVPR)}, Oct. 2015, pp. 1000--1008.

\bibitem{shi2016deep}
W.~Shi, C.~C. Loy, and X.~Tang, ``Deep specialized network for illuminant
  estimation,'' in \emph{Eur. Conf. Comput. Vis.}, 2016, pp. 371--387.

\bibitem{luke2017augmentation}
\BIBentryALTinterwordspacing
L.~Taylor and G.~Nitschke, ``Improving deep learning using generic data
  augmentation,'' 2017. [Online]. Available:
  \url{http://arxiv.org/abs/1708.06020}
\BIBentrySTDinterwordspacing

\bibitem{mackintosh2012zxing}
A.~Mackintosh, A.~Martin, B.~Brown, C.~Brunschen, and S.~Daniel, ``Zxing. open
  source library to read {1D}/{2D} barcodes,'' 2012.

\bibitem{srivastava2007bayesian}
S.~Srivastava, M.~R. Gupta, and B.~A. Frigyik, ``Bayesian quadratic
  discriminant analysis,'' \emph{J. Mach. Learn. Res.}, vol.~8, no.~6, pp.
  1277--1305, Dec. 2007.

\bibitem{criminisi2011decision}
A.~Criminisi, J.~Shotton, and E.~Konukoglu, ``Decision forests for
  classification, regression, density estimation, manifold learning and
  semi-supervised learning,'' \emph{Microsoft Research Cambridge, Tech. Rep.
  MSRTR-2011-114}, vol.~5, no.~6, p.~12, 2011.

\bibitem{rahimi2007random}
A.~Rahimi and B.~Recht, ``Random features for large-scale kernel machines,'' in
  \emph{Advances in neural Information processing systems}, 2007, pp.
  1177--1184.

\bibitem{vedaldi2012efficient}
A.~Vedaldi and A.~Zisserman, ``Efficient additive kernels via explicit feature
  maps,'' \emph{{IEEE} Trans. Pattern Anal. Mach. Intell.}, vol.~34, no.~3, pp.
  480--492, Mar. 2012.

\bibitem{chang2010training}
Y.-W. Chang, C.-J. Hsieh, K.-W. Chang, M.~Ringgaard, and C.-J. Lin, ``Training
  and testing low-degree polynomial data mappings via linear svm,'' \emph{J.
  Mach. Learn. Res.}, vol.~11, pp. 1471--1490, Apr. 2010.

\end{thebibliography}

\end{document}